\documentclass{article}
\usepackage{graphicx} 

\title{ Differentially Private Kernelized Contextual Bandits}
 \author{	
	 Nikola Pavlovic\thanks{
         Department of Electrical and Computer Engineering,
         Cornell University;
         \texttt{\{np358,qz16\}@cornell.edu}.} \\
         Cornell     \\
	 \and
     Sudeep Salgia\thanks{
         Department of Electrical and Computer Engineering,
         Carnegie Mellon University;
         \texttt{ssalgia@andrew.cmu.edu}.}   \\
	 Carnegie Mellon   \\
     \and
	     Qing Zhao\footnotemark[1]\\
         Cornell   \\
 	} 
\date{January 2025}

\usepackage[margin=1in]{geometry}

\usepackage[utf8]{inputenc} 
\usepackage[T1]{fontenc}    

\usepackage[unicode=true,
 bookmarks=false,
 breaklinks=false,pdfborder={0 0 1},colorlinks=false]
 {hyperref}
\hypersetup{
 colorlinks,citecolor=blue,filecolor=blue,linkcolor=blue,urlcolor=blue}
\usepackage{url}            
\usepackage{booktabs}       
\usepackage{amsfonts}       
\usepackage{nicefrac}       
\usepackage{microtype}      
\usepackage{xcolor}         
\usepackage{bm}
\usepackage{amsmath}
\usepackage{amssymb}

\usepackage{multirow}
\usepackage{colortbl}
\definecolor{Gray}{gray}{0.85}
\usepackage{enumitem}

\usepackage{microtype}
\usepackage{graphicx}
\usepackage{subfig}
\usepackage{booktabs} 
\usepackage{verbatim}
\usepackage{amsmath}
\usepackage{mathtools}
\usepackage{dsfont}
\usepackage{amsthm}
\usepackage{lipsum}
\usepackage{mathrsfs}
\usepackage{algorithm}
\usepackage{algorithmic}

\usepackage{graphicx}

\usepackage{amssymb}
\usepackage{hyperref}
\usepackage{titling}

\usepackage{amsthm}
\usepackage{comment}
\usepackage{mathtools}

\usepackage{natbib}

\newcommand{\comm}[1]{}

\newcommand{\hatZ}{\widehat{\mathbf{Z}}}
\newcommand{\tildeZ}{\widetilde{\mathbf{Z}}}

\newtheorem{theorem}{Theorem}[section]

\newtheorem{lemma}[theorem]{Lemma}

\theoremstyle{definition}
\newtheorem{definition}[theorem]{Definition}
\newtheorem{assumption}[theorem]{Assumption}

\theoremstyle{remark}

\usepackage[utf8]{inputenc} 
\usepackage[T1]{fontenc}

\usepackage[textsize=tiny]{todonotes}

\newcommand{\E}{\mathbb{E}}

\newcommand{\N}{\mathbb{N}}

\newcommand{\R}{\mathbb{R}}

\newcommand{\cA}{\mathcal{A}}

\newcommand{\cC}{\mathcal{C}}
\newcommand{\cD}{\mathcal{D}}
\newcommand{\cE}{\mathcal{E}}

\newcommand{\cH}{\mathcal{H}}

\newcommand{\cM}{\mathcal{M}}

\newcommand{\cO}{\mathcal{O}}
\newcommand{\cP}{\mathcal{P}}

\newcommand{\cS}{\mathcal{S}}

\newcommand{\cU}{\mathcal{U}}

\newcommand{\cW}{\mathcal{W}}
\newcommand{\cX}{\mathcal{X}}
\newcommand{\cY}{\mathcal{Y}}
\newcommand{\cZ}{\mathcal{Z}}

\newcommand{\bfk}{\mathbf{k}}

\newcommand{\bfA}{\mathbf{A}}

\newcommand{\bfH}{\mathbf{H}}

\newcommand{\bfK}{\mathbf{K}}

\newcommand{\bfW}{\mathbf{W}}

\newcommand{\bfY}{\mathbf{Y}}
\newcommand{\bfZ}{\mathbf{Z}}

\newcommand{\sA}{\mathscr{A}}

\DeclareMathOperator*{\argmin}{arg\,min}

\newcommand{\ip}[2] {\langle #1, #2 \rangle }

\newcommand{\Id}{\textbf{Id}}

\begin{document}

\noindent

\noindent
\maketitle

\begin{abstract}
  We consider the problem of contextual kernel bandits with stochastic contexts, where the underlying reward function belongs to a known Reproducing Kernel Hilbert Space (RKHS). We study this problem under the additional constraint of joint differential privacy, where the agents needs to ensure that the sequence of query points is differentially private with respect to both the sequence of contexts and rewards. We propose a novel algorithm that improves upon the state of the art  and achieves an error rate of $\cO\left(\sqrt{\dfrac{\gamma_T}{T}} + \dfrac{\gamma_T}{T \varepsilon}\right)$ after $T$ queries for a large class of kernel families, where $\gamma_T$ represents the effective dimensionality of the kernel and $\varepsilon > 0$ is the privacy parameter. Our results are based on a novel estimator for the reward function that simultaneously enjoys high utility along with a low-sensitivity to observed rewards and contexts, which is crucial to obtain an order optimal learning performance with improved dependence on the privacy parameter.
\end{abstract}

\section{Introduction}


We study the problem of contextual kernel bandits, where an agent aims to maximize an unknown reward function based on noisy observations of the function at sequentially queried points. Specifically, at each time instant $t$, the agent is presented with a context $c_t \in \cC$, based on which it takes an action $x_t \in \cX$ and receives a noisy value of the reward $f(x_t, c_t)$. In this work, we consider the case where the reward function $f:\cX\times \cC \rightarrow \mathbb{R}$ belongs to a Reproducing Kernel Hilbert space (RKHS) of a known kernel $k$ and the contexts $c_t$ are drawn i.i.d. from a context distribution $\kappa$. Kernel bandits offer significantly more modelling capabilities compared to their linear counter-parts . In particular, it is known that the RKHS of typical kernels, such as the Mat\'ern family of kernels, can approximate almost all continuous functions on compact subsets of $\R^d$ \citep{Srinivas}. We measure the performance of an agent using expected predictive error rate, which is akin to the popular notion of simple regret, adapted for the contextual setting. In particular, after $T$ rounds of interaction, let $\widehat{x}_T(c_{T+1})$ denote the output of the algorithm $\sA$ for an observed context $c_{T+1}$. Then the error rate of the algorithm is defined as
\begin{align}
    \textsc{ER}(\sA) &= \mathop{\E}_{c_{T+1} \sim \kappa}\left[\sup_{x \in \cX} f(c_{T+1}, x) - f(c_{T+1}, \widehat{x}_T(c_{T+1}))\right]. \label{eqn:er_def}
\end{align}
The learning objective of the agent is to minimize the worst-case error rate over the class functions with a given bounded RKHS norm.

\subsection{Private Kernel Bandits}\label{sec:private_kernel_badnits}

In many applications, the contexts and the rewards may carry sensitive information, that might be inadvertently revealed by the algorithm through its choice of query points. For example, consider the problem of learning a recommendation system for an online shopping platform. At each time instant, the learning agent observes a random user along with their associated information, e.g., their search and purchase history, and then chooses a product to recommend and subsequently observes whether the user interacts with the recommended item. The objective for the learning agent is to design a predictor with a low error rate on a newly observed user from the distribution. The analogy to the contextual bandit setting is almost immediate --- the user associated information serves as a context, the recommended product is the action and the user feedback is the reward. While the user related information,  and user feedback, provide valuable data for learning a good predictor, they contain private information about the user that needs to be protected.

This motivates  the problem of contextual kernel bandits under privacy constraints. We adopt the framework of joint differential privacy (JDP)\citep{ShariffandSheffet,DubeyPentland2020}, where we seek to design algorithms that are  differentially private with respect to both the context and the reward sequence (See Section~\ref{JDP} for a precise definition).
The primary challenge in designing differentially private learning algorithms is to balance the privacy-utility trade-off, i.e., to ensure meaningful learning while guaranteeing the privacy of the dataset. While there has been some effort towards designing differentially algorithm for multi-armed and linear bandits, the setting of kernelized bandits remains largely unexplored. Existing results on differentially private kernel bandits either apply only to a small class of kernel families or adopt a weaker notion of privacy (See Section~\ref{related_works:kernel_privacy} for additional discussion). In particular, there are no  differentially private algorithms that achieve diminishing error rate under  JDP for the commonly used kernel families, e.g., Mat\'ern kernels.

\subsection{Main Results}

We propose the first algorithm for contextual kernel bandits that is jointly differentially private with respect to the contexts and the rewards and theoretically guarantees a diminishing simple regret for all kernels with polynomially decaying eigen-values. Kernels with polynomial eigen decay include the class of commonly used kernels like Mat\'ern and Square exponential kernels. In particular, we establish a worst case error rate of $\cO\left(\sqrt{\gamma_T/T}+\gamma_T/(T\varepsilon)\right)$, where $\gamma_T$ is the information gain and represents the effective dimensionality of the kernel and $\varepsilon$ is the privacy parameter. In the non-private setting i.e., $\varepsilon \to \infty$, this reduces to an error rate of $\cO(\sqrt{\gamma_T/T})$, which is known to be order-optimal \citep{Scalett_et_al_2017}. The best current simple regret upper bound In the private setting, derived from the  bound for the cumulative regret, is $\cO(\sqrt{\gamma_T/T}+\sqrt{\gamma_T/(T\varepsilon}))$. Notably this bound only applies to Square exponential kernels.

In sequential learning problems, the dataset with respect to which the algorithm needs to guarantee privacy continues to expand as new data points arrive over time. Consequently, this forces the algorithms to add a small additional layer of privacy for each query point. This injection of additional privacy induced noise at each time instant, in absence of a careful control, compounds over time and leads to poor utility. For the problem of private linear bandits \citep{ShariffandSheffet,DubeyPentland2020}, existing studies avoid this pitfall through the use of the tree-based mechanism 
\citep{Dwork}. Tree-based mechanism allows for online release of prefix-sums, where the noise added to the sums does not scale linearly but rather logarithmically; thereby alleviating the noise compounding effect. For the problem of linear bandits, the key challenge is to privatize the covariance matrix for the ridge regression estimate. This is solved by noting that the covariance has an additive structure (sum of rank-one matrices) in the feature space, thereby allowing the use of tree-based mechanism. Taking this approach ensures that the privacy based error grows logarithmically, instead of linearly, resulting in meaningful utility bounds. However, in kernel based bandits, the features belong to an infinite-dimensional space which renders the use of this technique infeasible.

The work of \citep{Dubey2021}, which is the current state of the art, circumvents this issue by approximating the kernel with a low-dimensional surrogate, which allows them to reduce the problem to that of a finite-dimensional linear bandit. However, this approach is limited to Square Exponential kernels. It is not clear how such an approach can be applied to other kernels, e.g. the Mat\'ern family of kernels.

The proposed technique in this work offers a departure from prevailing approaches to resolve the pitfalls of ensuring privacy.
The common approach adopted by existing algorithms to address the privacy-utility trade-off is to privatize a high utility estimator. Namely, \citep{ShariffandSheffet, Dubey2021} take a UCB-based estimator with order optimal non-private performance, and through careful addition of noise ensure privacy constraints while not significantly degrading the  performance of the algorithm.

This utility first approach leads to a query strategy that is highly adaptive to the observed history in order to maximize the utility. However, high adaptivity results in high sensitivity to data points which makes preserving privacy much harder. 

In this work, we adopt a different approach, where we put the privacy constraint at the forefront and then optimize utility.  To this effect, there are two key components to our algorithm design that ensure the privacy constraint. The first component is the uniform sampling of the query points from the action set $\cX$, independent of the context-reward observations. This \emph{data-independent}, \emph{non-adaptive} query strategy,  decouples the query points from the context-reward pairs and immediately guarantees privacy during the learning process. Our approach solves the problem of compounding noise during learning by simply ensuring the query points are private by design. There is thus no need to add any noise during the learning stage.

The second component is the design of a  novel low -sensitivity estimator to be used for the final prediction. In current kernel bandit literature the posterior-mean estimator (see eq.(\ref{eqn:posterior_mean})) is used almost exclusively. Although the posterior-mean offers order optimal  approximation error of the reward function, it carries a strong dependence on the dataset through the reward and feature vector $\bfk_{\bfW_T}(\cdot)$ as well as the Gramian matrix $\bfK_{\bfW_T,\bfW_T}$. As the prediction of the algorithm also needs to be private, it is necessary for the estimator to have both high utility and small sensitivity with respect to the dataset. This is seemingly an infeasible requirement as the small error requires adaptivity while the small sensitivity requires a weak dependence on the dataset.

By combining the recent advancements in non-private kernel bandits \citep{Sudeep_Uniform_Sampling} along with a novel technique for covariance estimation we design an estimator that only depends on the dataset through the feature and reward vector (please see eq.(\ref{eqn:estimator})) while retaining the approximation error of the posterior mean estimator. The essence of our proposed approach is to replace the covariance corresponding to a set of randomly sampled actions with that of an independently drawn set of actions. We use concentration results to establish that our new estimator offers the same order of approximation error as the posterior mean. At the same time, the independence between the set of samples and the dataset allows it to also enjoy low sensitivity.

\subsection{Related Work}

\paragraph{Kernel-based bandits.} The problem of kernel-based bandit optimization has been extensively studied in the non-private setting. Starting with the seminal work of \citep{Srinivas}, numerous algorithms for kernel-based have been proposed in both contextual\citep{Valko_et_al_2013} and non-contextual settings \citep{Batched_Communication,GP_ThreDS}. The optimal performance in the non-private setting is well-understood where several algorithms \citep{Batched_Communication,Valko_et_al_2013, GP_ThreDS} are known to achieve the order-optimal performance that matches the lower bound \citep{Scalett_et_al_2017}.

\paragraph{Private Bandit optimization.} The problem of differentially private bandit optimization has received considerable attention for both multi-armed and linear bandits.  The problem of linear bandits with JDP was first introduced by \citep{ShariffandSheffet} where they propose an algorithm that achieves a cumulative regret of $\tilde\cO({\sqrt{dT}/\varepsilon})$. \citep{DubeyPentland2020} extend their results to the distributed setting where they achieve a cumulative regret of $\cO(M^{3/4}\sqrt{T/\varepsilon})$\footnote{This is rectified regret bound from \citep{Zhou_at_al_2023} }. \citep{Garcelon_et_al_2022} consider the shuffle model of privacy as midway point between the central (JDP) privacy and (LDP) local privacy. The LDP notion of privacy is a  more challenging setup of differential privacy , where the users do not trust the algorithm and all the data needs to be privatized before leaving the user. In contrast (JDP), requires the data to be privatized only before being used, allowing for processing in batches as in \citep{ShariffandSheffet}. \citep{Zheng_et_al_2020} studies the problem of generalized linear bandits under an LDP constraint and propose an algorithm with a cumulative regret $\tilde\cO(T^{3/4}/\varepsilon)$. \citep{Huang_et_al2024}  study differential privacy for distributed contextual linear bandits in both central and local setting. They provide lower bounds for both settings, with a matching upper bound in the case of central differential privacy. \citep{Hanna_et_al_2024} study the LDP, CDP and shuffle model for linear non-contextual bandits. Only the rewards are considered to be sensitive data.\\
Privacy constraints have also been studied in the multi-arm-bandit framework(MAB). \citep{Azize_Basu_2022,Azize_et_al_2024} explore globally private best arm-indentification while \citep{Tenenbaum_et_al_2021} studied cumulative regret optimization under shuffle model of privacy.\\


\label{related_works:kernel_privacy}
The problem of private kernel bandits was first studied by \citep{Kusner_et_al_2015} in the context of hyper-parameter tuning where the authors focus only on privatizing the final query point as opposed to  all of them. \citep{Kharkovskii_et_al_2020} study private kernel bandits for square exponential kernels under the setting where the algorithm and user are separate entities. The query points are required to be locally differentially private but not the rewards. Under the additional assumption of covariance matrix being diagonally dominant, they establish a simple regret of $\tilde\cO((\varepsilon^{-2}+\gamma_T/T)^{1/2})$ . \citep{Zhou_et_al_2021} consider kernel bandits with heavy tailed noise and only rewards are required to be private.  \citep{Dai_et_al_2021} study privacy for Thompson sampling in the problem of distributed  kernel bandits where the reward functions are assumed to be heterogeneous over users.\\

The work that is closest to ours is by \citep{Dubey2021} where the authors consider private contextual bandits. As mentioned earlier, they approximate the kernel using a low dimensional surrogate after which they use techniques from private linear bandits to design an algorithm for the kernel-based problem. However, the result in \citep{Dubey2021} crucially depends on the assumption that underlying kernel can be approximated by a features whose dimension is at most polylogarithmic in $T$ and has a separable Fourier transform. Such an assumption is only satisfied for the family of Squared Exponential kernels and it is not obvious how to extend this approach for more general kernels.

\section{Problem Formulation and Preliminaries}\label{Formulation}

\subsection{RKHS, Mercer's Theorem and GP Models}

Consider a positive definite kernel $k: \cW \times \cW \to \R$, where $\cW$ is a compact set in a given metric space. A Hilbert space $\cH_k$ of functions on $\cW$ equipped with an inner product $\ip{\cdot}{\cdot}_{\cH_k}$ is called a Reproducing Kernel Hilbert Space (RKHS) with reproducing kernel $k$ if the following conditions are satisfied: (i) $\forall \ w \in \cW$, $k(\cdot, w) \in \cH_k$; (ii) $\forall \ w \in \cW$, $\forall \ f \in \cH_k$, $f(w) = \ip{f}{k(\cdot, w)}_{\cH_k}$. The inner product induces the RKHS norm, $\|f\|_{\cH_k}^2 = \ip{f}{f}_{\cH_k}$. We use $\phi(w)$ to denote $k(\cdot, w)$ and WLOG assume that $k(w,w) = \|\phi(w)\|_{\cH_k}^2 \leq 1$. 





Let $\zeta$ be a finite Borel probability measure supported on $\cW$ and let $L_2(\zeta,\cW)$ denote the Hilbert space of functions that are square-integrable w.r.t. $\zeta$. Mercer’s Theorem provides an alternative representation for RKHS through the eigenvalues and eigenfunctions of a kernel integral operator defined over $L_2(\zeta,\cW)$ using the kernel $k$.

\begin{theorem}\citep{SVM_Book}
Let $\mathcal{W}$ be a compact metric space and $k : \mathcal{W} \times \mathcal{W}\rightarrow \mathbb{R}$ be a continuous kernel. Furthermore, let $\zeta$ be a finite Borel probability measure supported on $\cW$. Then, there exists an orthonormal system of functions $\{\psi_j\}_{j\in \mathbb{N}}$ in $L_2(\zeta,\mathcal{W})$  and a sequence of non-negative values
$\{\lambda_j\}_{j\in \mathbb{N}}$ satisfying $\lambda_1\geq \lambda_2\dots \geq 0$ , such that $k(w,w')=\sum_{j\in \mathbb{N}}\lambda_j\psi_j(w)\psi_j(w')$ holds for all $w,w'\in \mathcal{W}$ and the convergence is absolute and uniform over $w,w'\in \mathcal{W}$. 
\end{theorem}

Consequently, the Mercer representation\cite[Thm. 4.51]{SVM_Book} of the RKHS of $k$ is given as
\begin{align*}
    \cH_k = \bigg\{ f := \sum_{j \in \N} \alpha_j {\lambda_j}^{\frac{1}{2}} \psi_j : \|f\|_{\cH_k}^2 = \sum_{j \in \N} \alpha_j^2 < \infty \bigg\}.
\end{align*}

A commonly used technique to characterize a class of kernels is through their eigendecay profile. 
\begin{definition}\label{assumptio:polynomial_decay}
    Let $\{\lambda_j\}_{j\in \mathbb{N}}$ denote the eigenvalues of a kernel $k$ arranged in the descending order. The kernel $k$ is said to satisfy the polynomial eigendecay condition with a parameter $\beta_p>1$ if, for
    some universal constant $C_p>0$, the relation $\lambda_j\leq C_pj^{-\beta_p}$ holds for all $j \in \N$ .
\end{definition}

We make the following assumption on the kernel $k$ and the eigenfunctions $\{\psi_j\}_{j \in \N}$, which is commonly adopted in kernel-based optimization literature \citep{information_gain_bound,Chatterji_et_al_2019,Riutort_Mayol_et_al_2023, Whitehouse_et_al_2023}.

\begin{assumption}\label{assumption:bounded_eigen_functions}
We assume that the kernel $k$ satisfies the polynomial eigendecay condition with  parameter $\beta_p >1$. The eigen-functions $\{\psi_j\}_{j \in \mathbb{N}}$ corresponding to the kernel $k$ are continous and hence bounded on $\cW$ i.e $\exists \ F > 0,$ such that $ \sup_{w\in \cW}|\psi_j(w)|\leq F$ for all $j\in\mathbb{N}$.
\end{assumption}





A Gaussian Process (GP) is a random process $G$ indexed by $\cW$ and is associated with a mean function $\mu : \cW \to \R$ and a positive definite kernel $k : \cW \times \cW \to \R$. The random process $G$ is defined such that for all finite subsets of $\cW$,  $\{w_1, w_2, \dots, w_T\} \subset \cW$, $T \in \mathbb{N}$, the random vector $[G(w_1), G(w_2), \dots, G(w_T)]^{\top}$ follows a multivariate Gaussian distribution with mean vector $[\mu(w_1), \dots, \mu(w_T)]^{\top}$ and covariance matrix $[k(w_i, w_j)]_{i,j=1}^T$. Throughout the work, we consider GPs with $\mu \equiv 0$. When used as a prior for a data generating process under Gaussian noise, the conjugate property provides closed form expressions for the posterior mean and covariance of the GP model. Specifically, given a set of observations $\{\mathbf{W}_T,\mathbf{Y}_T\} = \{(w_i,y_i)\}_{i=1}^T$ from the underlying process, the expression for posterior mean and variance of GP model is given as follows:
\begin{align}
    \mu_{T}(w) & =\bfk^{\top}_{\mathbf{W}_T}(w)(\tau \mathbf{I}_T+\mathbf{K}_{\mathbf{W}_T,\mathbf{W}_T})^{-1}\mathbf{Y}_T, \label{eqn:posterior_mean}\\
    \sigma^2_T(w)& =k(w,w)-\bfk_{\mathbf{W}_T}^{\top}(w)(\tau\mathbf{I}_T+\mathbf{K}_{\mathbf{W}_T,\mathbf{W}_T})^{-1}k_{\mathbf{W}_T}(w). \label{eqn:posterior_variance}
\end{align}
In the above expressions, $\bfk_{\mathbf{W}_T}(w)=[k(w_1,w),k(w_2,w)\dots k(w_T,w)]^{\top}$, $\mathbf{K}_{\mathbf{W}_T,\mathbf{W}_T}=[k(w_i,w_j)]_{i,j=1}^{T}$, $\mathbf{I}_T$ is the $T \times T$ identity matrix and $\tau$ is the variance of the Gaussian noise. 

Following a standard approach in the literature \citep{Srinivas}, we model the data corresponding to observations from the unknown $f$, which belongs to the RKHS of a positive definite kernel $k$, using a GP with the same covariance kernel $k$. In particular, we assume a \emph{fictitious} GP prior over the fixed, unknown function $f$ along with \emph{fictitious} Gaussian distribution for the noise. Such a modelling allows us to predict the values of $f$ and characterize the prediction error through the posterior mean and variance of the GP model.


 \label{gamma_explained}
Lastly, given a set of points $\mathbf{W}_T = \{w_1, w_2, \dots, w_T\} \in \cW$, the information gain of the set $\mathbf{W}_T$ is defined as $\gamma_{\mathbf{W}_T} := \frac{1}{2} \log(\det(\mathbf{I}_T + \tau^{-1}\mathbf{K}_{\mathbf{W}_T,\mathbf{W}_T}))$. Using this, we can define the maximal information gain of a kernel as $\gamma_T := \sup_{\mathbf{W}_T \in \cW^T} \gamma_{\mathbf{W}_T}$. Maximal information gain is closely related to the effective dimension of a kernel \citep{Calandriello_Sketching} and helps characterize the regret performance of kernel bandit algorithms \citep{Srinivas,Gopalan_2017}. $\gamma_T$ depends only the kernel and $\tau$ and has been shown to be an increasing sublinear function of $T$ \citep{Srinivas, information_gain_bound}.

\subsection{Joint Differential Privacy}\label{JDP}

We adopt the framework of Joint Differential Privacy presented in \citep{ShariffandSheffet}. Let $\cS_T=\{(c_1,y_1),(c_2,y_2),\dots (c_T,y_T),c_{T+1}\}$, referred to as a database, denote the collection of all contexts and rewards seen in the duration of the algorithm. Here, $c_{T+1}$ denotes the contexts drawn at the evaluation instant $T+1$.

\begin{definition}
Two databases $\cS_T,\cS'_T$ are said to be $t$-neighbours if they only differ in the context and reward at time $t$ . Specifically,  the databases $\cS_T=\{(c_1,y_1),(c_2,y_2),\dots (c_t,y_t),\dots (c_T,y_T),c_{T+1}\}$ and $\cS'_T=\{(c_1,y_1),(c_2,y_2),\dots (c'_t,y'_t),\dots (c_T,y_T),c_{T+1}\}$ are considered to be $t$-neighbours.
\end{definition}

JDP ensures that a malicious adversary cannot confidently differentiate  between the agent database and any of its neighbours. This constraint can be mathematically presented as:

\begin{definition}\label{Def:JDP} 

A randomized algorithm $\sA$ is $(\varepsilon,\delta)$-joint differentially
private (JDP) under continual observation if for all $t \leq T$ and all pairs of $t$-neighboring databases $\mathcal{S}$ and
$\mathcal{S}^{'}$ and any subset $\cP_{>t} \subset \mathcal{X}^{T-t+1}$ of sequence of points ranging from day $t+1$ to $T+1$  \footnote{counting the final $\widehat{x}_T(c_{T+1})$ for which there is no feedback}, it holds that:
\[
\Pr\left(\sA(\mathcal{S})\in \cP_{>t}\right)\leq e^{\varepsilon} \cdot \Pr\left(\sA(\mathcal{S}')\in \cP_{>t}\right)+\delta,
\]
where the probability is taken over the random coins generated by the algorithm.
\end{definition}

Note that JDP does not require the query point $x_t$ to be private with respect to the current context and reward $(c_t,y_t)$, but only with respect to previously seen contexts and rewards $\cS_{<t}$. As shown in \cite[Claim 13]{ShariffandSheffet} requiring privacy with respect to current context and reward would lead to provably $\cO(1)$ simple regret performance.

\subsection{Problem statement}\label{Problem Statement}

We consider the problem of contextual kernel bandits, where at each time instant $t$, the learning agent is presented with a context $c_t \in \cC$ based on which it queries a point $x_t \in \cX$ and receives a noisy reward $y_t=f(x_t,c_t)+ \eta_t$, where $\eta_t$ denotes the noise. We assume that the sets $\cX \in \R^d$ and $\cC \in \R^{d'}$ are compact and convex. The reward function $f$ belongs to a RKHS corresponding to a kernel $k$ defined over $\cW := \cX \times \cC$.
We consider the setting where the contexts are drawn i.i.d. across time according to some context distribution $\kappa$. The contextual bandits with stochastic contexts has widely studied in the literature \citep{Huang_et_al2024,Han_et_al_2021,Amani_et_al_2023, Hanna_et_al_2022, Hanna_et_al_2023}. We make the following assumptions that are commonly adopted in the literature.

\begin{assumption}\label{assumption:Sub_Gaussian}
The noise term $\eta_t$ is assumed to be i.i.d across all time instances and is a zero-mean, $R$ sub-Gaussian random variable i.e.,  it satisfies the relation $\mathbb{E}[\exp(q\eta )]\leq \exp(q^2R^2/2)$ for all $q \in \mathbb{R}$.
\end{assumption}

\begin{assumption}\label{assumption: bounded_rewards}
    The rewards $\{y_t\}_{t=1}^{T}$ in the duration of the algorithm are bounded in absolute values, $|y_t|<B,\forall t\leq T$
\end{assumption}


Assumption(\ref{assumption: bounded_rewards}) is adopted across privacy literature \citep{DubeyPentland2020,ShariffandSheffet,Han_et_al_2021,Zheng_et_al_2020}
, and ensures that an adversary cannot probe an unbounded reward as an input to the algorithm. We note that that this assumption could be removed by simple clipping the rewards that have modulus higher than $B+ R\log(T/\delta)$. By sub-gaussian assumption on the rewards all rewards would remain unchanged, with probability $1-\delta$.

\begin{assumption}\label{assumption:Lipschitz}
    We assume the reward function $f$ is $L_{f}$-Lipschitz, i.e., the following relation holds for all $w, w' \in \cW = \cX \times \cC$
    \begin{align*}
    |f(w)-f(w')|\leq L_f\|w-w'\|_2
    \end{align*}
\end{assumption}


\begin{assumption}\label{assumption:Grid}
For each $r \in \mathbb{N}$, there exists a discretization $\cU_r$ of $\mathcal{W}$ with $|\cU_r| = \mathrm{poly}(r)$\footnote{The notation $g(x) = \mathrm{poly}(x)$ is equivalent to $g(x) = \cO(x^k)$ for some $k \in \mathbb{N}$.} such that, for any $f \in \mathcal{H}_k$, we have $|f(w)-f([w]_{\cU_r})|\leq \frac{\|f\|_{\mathcal{H}_k}}{r}$, where $[w]_{\cU_r} =\argmin_{w'\in \cU_r} \|w-w'\|_2$. 
\end{assumption}

\begin{assumption}
    We have a context generator that is able to generate contexts i.i.d according to the distribution $\kappa$.
    \label{assumption:contexts}
\end{assumption}

Assumptions~\ref{assumption:Sub_Gaussian} and~\ref{assumption:Lipschitz} are mild assumptions that are commonly adopted in the literature \citep{Srinivas, Gopalan_2017, Batched_Communication, GP_ThreDS,Lee_at_al_2022}. For commonly used kernels kernels like Squared Exponential and Mat\'ern kernels elements of its RKHS are known to be Lipschitz continuous \citep{Lee_at_al_2022}. Assumption \ref{assumption:Grid} is nearly universally used to apply the confidence bounds on the continuous domain \citep{Vakili_Aprox_Conv,Vakili_Kernel_Simple_Regret,Batched_Communication,Sudeep_Uniform_Sampling}\\ 

Assumption~\ref{assumption:contexts} is milder than the assumption of having complete knowledge of context distribution, an assumption that has been commonly adopted in several existing studies on stochastic contextual bandits\citep{Amani_et_al_2023,Hanna_et_al_2022,Hanna_et_al_2023}




\noindent In our analogy to the shopping platform in the introduction, these contexts are  not recorded from the arrivals of any "real" users but is rather an assumption that the platform can learn its user demographic(location, common searches etc.). In other words these contexts are not the input to the  algorithm but are rather generated by the algorithm's random coin.
We also emphasize these "artificially generated"  contexts are not a result of interaction with the environment. Consequently, no feedback is associated with the generated contexts and thus they carry no information on the reward function and cannot offer any trivial advantages in the learning.\\

\section{Algorithm Description}\label{Algorithm_Description}

\begin{algorithm}[H]
\caption{USCA}\label{alg:Algorithm 3}
\begin{algorithmic}[1]
    \STATE \textbf{Input}:  error probability $\delta$, privacy budget $\varepsilon$,
    \STATE Initialize $\mathbf{W}_{T},\mathbf{Y}_T, \cZ \leftarrow \emptyset$ 
    \STATE \texttt{// Learning stage}
    \FOR{$t = 1,2, \dots, T$}
        \STATE Receive the context vector $c_t$ 
        \STATE Query $x_t$ from $\cX$ uniformly at random and observe the reward $y_t$
        \STATE $\mathbf{W}_{T} \leftarrow \mathbf{W}_{T}\cup \{(x_t,c_t)\}, \bfY_T \leftarrow \bfY_T\cup\{y_t\}$
    \ENDFOR
    \FOR{$k = 1,2, \dots, TK$}
        \STATE Sample $c_{k}$ from the context measure $\kappa$
        \STATE Sample $x_{k}$ from $\cX$ uniformly at random
        \STATE $\cZ \leftarrow \cZ \cup \{(x_{k},c_{k})\}$
    \ENDFOR
    \STATE Construct the posterior mean $\overline{\mu}_T,\overline{\sigma}$ with an approximating set $\cZ$ and dataset $\bfW_T$ from eq.(\ref{eqn:estimator}, \ref{eqn:variance})
    \STATE \texttt{// Prediction stage}
    \STATE Observe the context set $c_{T+1}$
    \STATE Sample 
    $$\widehat{x}_T(c_{T+1}) \sim \mathcal{E}(\overline{\mu}_T(\cdot,c_{T+1}),\varepsilon, 2B\sup\overline{\sigma}^2)$$
    where $\cE$ is as defined in~ Def.(\ref{def:exp_measure})
    \STATE \textbf{Output} $\widehat{x}_T(c_{T+1})$
\end{algorithmic}
\end{algorithm}

In this section we present our algorithm, \textsc{USCA}. The main features of the algorithm are  \textbf{U}niform \textbf{S}ampling with \textbf{C}ovariance \textbf{A}pproximation. We seperate $\textsc{USCA}$ into two stages, reflecting the nature of the agents interaction with the environment and not a design philosphy. In the learning stage agent collects the information on the reward function through the collected feedback, $\bfY_T$. In the prediction stage, the agent uses the collected data to project the best performing point for a given context $c_{T+1}$.

During learning stage, our algorithm adopts a data-independent, random sampling approach. After being presented with a context vector $c_t$ agent draws a query point $x_t$ from $\cX$ uniformly at random, independently from previous contexts and rewards. This, data invariant, sampling method ensures that there is no loss of privacy during learning and thus no noise injection is necessary at this stage.\\

A widely accepted approach in kernel bandit learning  is to sample points from the domain based on the posterior statistics in eqs.(\ref{eqn:posterior_mean}, \ref{eqn:posterior_variance}). \textsc{USCA} abandons this approach and utilizes novel estimator $\overline{\mu}_T$ and surrogate variance $\overline{\sigma}^2$. Estimator $\overline{\mu}_T$ is obtained by approximating the covariance matrix of the sampled points (lines 4-7) by an empirical average. The empirical average is calculated based on the points in the approximating set $\cZ$. In the context of our algorithm $\overline{\sigma}^2$, is seen as the posterior variance of a GP process after sampling the approximating set $\cZ$(see eq.\ref{eqn:posterior_variance}) . As such, $\overline{\sigma}$ is independent of any contexts or reward seen by the agent, and only depends on the points of the approximating set $\cZ$.  \\

To construct  $\overline{\sigma}, \overline{\mu}_T$ we need a sample of $TK$ context vectors  drawn from  the context distribution $\kappa$ where $K=\lceil T/\gamma_T\rceil$. We emphasize that these contexts are  sampled independently from the context space $\cC$ and are not an input to the algorithm but are rather generated by the algorithm's random coin. Following the generation of the context vector $c_{k}$ the agent samples $x_{k}$ from the domain $\cX$ uniformly at random and the sample $(x_{k},c_{k})\in \cW$ is added to the approximating set $\cZ$ that we use to calculate $\overline{\mu}_T,\overline{\sigma}^2$ .

Next, we calculate the statistics $\overline{\sigma},\overline{\mu}_T:\cW \rightarrow \mathbb{R}$ from the sampled approximating set(lines 11-15) $\cZ$ and the data set $(\bfW_T,\bfY_T)$ accrued during learning (lines 4-7). We introduce the parametric form of the estimator $\overline{\mu}_T$ in Lemma \ref{approximation_lemma}, proving it sufficiently approximates the posterior mean $\mu_T$(see eq.(\ref{eqn:posterior_mean})). However in this form $\overline{\mu}_T$ is computationally intractable . We apply the celebrated kernel trick \citep{SVM_Book} to obtain a computationally more favorable form:
\begin{lemma}
    The statistics $\overline{\mu}_T,\overline{\sigma}$ in \textsc{USCA} can be calculated as:
    \begin{align}
    \overline{\mu}_T(w)&=\frac{1}{\tau}\left(k_{\mathbf{W}_T}(w^{\top})\bfY_T-\bfk^{\top}_{\cZ}(w)\left(\bfK_{\cZ, \cZ}+K\tau\mathbf{I}_{\cZ}\right)^{-1}\bfK_{\cZ,\bfW_T}\bfY_T\right)\label{eqn:estimator}\\
    \overline{\sigma}^2(w)&=\frac{1}{\tau}\left(k(w,w)-\bfk_{\cZ}^{\top}(w)\left(\bfK_{\cZ,\cZ}+K\tau\mathbf{I}_{\cZ}\right)^{-1}k_{\cZ}(w)\right)\label{eqn:variance}
    \end{align}

    Where $\bfK_{\cZ,\bfW_T}=\{k(a,b)\}_{a\in \cZ,b\in \bfW_T}, \bfK_{\cZ,\cZ}=\{k(a,b)\}_{(a,b)\in \cZ^2}$ and $K=\lceil T/\gamma_T\rceil$ .
\end{lemma}

\begin{proof}
    For a proof please see Lemma \ref{appendix:kernel_trick}.
\end{proof}

After $T$-time instances the agent is provided with a final context vector $c_{T+1}$ and has to output the point $\widehat{x}_T(c_{T+1})$ that should maximize the reward for $c_{T+1}$.\\
To privatize the final output agent samples $\widehat{x}_T(c_{T+1})$ according to exponential measure, with the exponent proportional to $\overline{\mu_T}$. More specifically:

\begin{definition}\label{def:exp_measure}
Define the measure $\cE(\overline{\mu}_T,\varepsilon,m)$ on $\cX$ as :
\begin{align*}
\mathcal{E}(\overline{\mu}_T(r),\varepsilon,m)= \frac{\exp(\overline{\mu}_T(r)\varepsilon/(2m))}{\int_{\cX}\exp(\overline{\mu}_T(r)\varepsilon/(2m))\nu_0(dr)}
\end{align*}

Where $\nu_0$ is the Lebesgue measure over $\cX$.

\end{definition}

In privacy literature the method of sampling from $\cE$ is known as the exponential mechanism \citep{McSherryadnTalwar, Dwork}.
Using the closed  form expression for the estimator  $\overline{\mu}_T$ and surrogate variance $\overline{\sigma}$ in eqs.(\ref{eqn:estimator},\ref{eqn:variance}) agent samples $\widehat{x}_{T}(c_{T+1})$ from $\cX$ according to  $\mathcal{E}\left(\overline{\mu}_T(\cdot,c_{T+1}),\varepsilon, 2B\sup\overline{\sigma}^2\right)$.



\section{Performance Analysis}\label{Performance_Analysis}

The following theorem characterizes the performance of our proposed algorithm, \textsc{USCA}.

\begin{theorem}\label{main_text:theorem}
Consider the contextual kernelized bandits problem described in Sec.~\ref{Problem Statement} where the underlying kernel function satisfies Assumption~\ref{assumptio:polynomial_decay} with parameter $\beta_p>1$ and the reward function $f$ is $L_f$-Lipschitz. If the USCA algorithm is run for $T$ steps with a privacy parameter $\varepsilon$, then for all $\varepsilon > 0$, $\delta \in (0,1)$ and $T > T_0$, 
\begin{itemize}
    \item USCA is $\varepsilon$-JDP;
    \item The error rate of USCA satisfies the following relation with probability $1 - \delta$
    \begin{align*}
    \textsc{ER}(\textsc{USCA})&=\cO\left(\sqrt{\frac{\gamma_T}{T}}+\frac{1}{\varepsilon}\frac{\gamma_T}{T}\right).
\end{align*}
Here $T_0$ is the constant that depends on the kernel and the context distribution\footnote{Please refer to Theorem \ref{appendix:theorem} for an exact expression.} and probability in the error bound is taken over all the contexts, rewards and random coins of the algorithm.
\end{itemize}
\end{theorem}

As shown by the above theorem, \textsc{USCA} retains $\varepsilon$- JDP privacy while achieving diminising regret rate  of $\cO(\sqrt{\gamma_T/T}+\gamma_T/(T\varepsilon))$. We emphasize that although we state the final result in terms of expectation over the context vector $c_{T+1}$, as proven in the Theorem \ref{appendix:theorem}, the claim  holds uniformly over the entire context set $\cC$.

The key ingredient that allows USCA to achieve the performance guarantees outlined in Theorem~\ref{appendix:theorem} is the use of a novel reward estimator $\overline{\mu}_T$. The classical posterior mean estimate $\mu_T$ offers powerful predictive performance. However, characterizing the sensitivity of the estimator $\mu_T$ is challenging due to the non-linear relationship between the contexts, rewards and the predicted value. The problem is exacerbated by the fact that the sensitivity is defined in an adversarial sense, i.e., the differing context can be any value from the context set and are not necessarily drawn from the context distribution. This results in trivial bounds on the sensitivity of $\mu_T$ which prevents us from obtaining any meaningful utility guarantees. Our estimator $\overline{\mu}_T$ alleviates this issue by having a far weaker dependency on the the dataset $\cS_T$. Specifically, $\overline{\mu}_T$ is constructed by choosing a feature covariance matrix that is \emph{independent} of the context and reward sequence, which helps us significantly decrease the impact of a single point on the output and hence obtain meaningful sensitivity bounds. This idea is formalized in the following lemma:

\begin{lemma}\label{Lemma:sensitivity_bound}
Let $\cS_T,\cS^{'}_T$ be two $t$-neighbouring databases, and let $\overline{\mu}_T,\overline{\mu}'_T$ be the 2 estimator constructed for each of the databases . For $T>\overline{N}_1(\delta)$, we can bound the sensitivity of an estimator $\overline{\mu}_T$ in \textsc{USCA} as :
\begin{align*}
    \Delta\overline{\mu}_T&=\sup_{w\in \cW}\sup_{\cS_T,\cS_T^{'}\text{are t neigbours}}|\overline{\mu}_T(w)-\overline{\mu}'_T(w)|\leq \\
    &\leq 2B\sup_{w\in \cW} \overline{\sigma}^2(w)
\end{align*}
\end{lemma}

What makes $\overline{\mu}_T$ a particularly  powerful tool is the fact that in addition to guaranteeing low-sensitivity it also offers high predictive performance, as shown in the following lemma.

\begin{lemma}\label{main_text:lemma:confidence_bounds}
    For the estimator $\overline{\mu}_T$ introduced in eq. (\ref{eqn:estimator}), under the condition $T>\max(\overline{N}(\delta/4),\overline{N}_1(\delta/4))$ we claim with probability at least $1-\delta$:
    \begin{align*}
        \sup_{w\in\cW}|\overline{\mu}_T(w)-f(w)|=\cO\left(\sqrt{\frac{\gamma_T}{T}}\right)
    \end{align*}

    Where $\overline{N}_1(\delta),\overline{N}(\delta)$ are a $\delta$-dependent constants.
\end{lemma}

In particular, Lemma~\ref{main_text:lemma:confidence_bounds} states that $\overline{\mu}_T$ retains the order of approximation error of the posterior mean obtained in \citep{Vakili_Kernel_Simple_Regret}. This implies that predictive performance of $\overline{\mu}_T$ is the same as that of $\mu_T$ while offering reduced sensitivity. Note that even though $\overline{\mu}_T$ is constructed using a feature covariance matrix that is \emph{independent} of the context and reward sequence, the covariance matrix in $\overline{\mu}_T$ and $\mu_T$ are identically distributed. The independence allows us to obtain strong sensitivity bounds while the concentration properties  lead to similar predictive performance. We formalize this idea in the following novel spectral bound for empirical covariance matrix which may be of independent interest:

\begin{lemma}\label{main_text:spectral_bound}
Suppose $TK$ points $\{z_1,z_2, \dots z_{TK}\}$ are sampled i.i.d from $\cW$ according to a Borel measure $\varrho$. Let $\bfZ=T\mathbf{\Lambda} + \tau \mathbf{Id}$ where $\mathbf{\Lambda}=\mathbb{E}_{w\sim \varrho}[\phi(w )\phi(w)^{\top}]$ and define the operator:
\begin{align*}
    \tildeZ=\frac{1}{K}\sum_{i=1}^{TK} \phi(z_i)\phi(z_i)^{\top}+ \tau\mathbf{Id}
\end{align*}
By choosing parameter $K=\lceil T/\gamma_T\rceil$ under the condition that $T>\overline{N}_1(\delta)=(\frac{\log(1/\delta)}{16C_pF^2})^{2\beta_p/(\beta_p-1)}$ we claim with probability at least $1-\delta$:
\begin{align*}
    \|\tildeZ^{-1}\bfZ-\mathbf{Id}\|_2\leq  \frac{28}{17} \cdot \sqrt{\frac{\gamma_T}{T}\frac{\log(1/\delta)}{\tau}}
\end{align*}

\end{lemma}

Next we briefly explain  how the presented lemmas are utilized in the proof of Theorem \ref{main_text:theorem}.\\
With the sensitivity bounds obtained in Lemma \ref{Lemma:sensitivity_bound}, we use a similar approach to \citep[Lemma 7]{McSherryadnTalwar} to ensure that the final prediction $f(\widehat{x}_T(c_{T+1}))$ is close to optimal value $\sup_{x\in \cX} f(x,c_{T+1})$. In order to ensure the utility of continuous exponential mechanism the usual hurdle to overcome is lowed-bounding the volume of well performing points \citep{McSherryadnTalwar,Dwork}. To this end, we use a geometric Lemma \ref{appendix:geometric_lemma} that establishes a bound on volume, polynomial in the distance to the optimum.\\
Our approach to ensuring privacy follows the approach given in \citep{McSherryadnTalwar}, with the necessary modification for random functions. For a full proof please see Theorem \ref{Appendix:JDP_guarantee}.

\section{Conclusion}
We propose the first algorithm for contextual kernel bandits that is continually differentially private with respect to the contexts and the rewards and theoretically guarantees a diminishing simple regret for all commonly used kernels. In particular, we improve on the state of the art while  greatly expanding the set of admissible kernel families.\\

Key aspects of our approach are random sampling during learning and a novel high-utility, low-sensitivity estimator. As a theoretical contribution,  we propose a novel concentration result for the covariance matrices of RKHS elements, that could be of independent interest.

\newpage

\bibliographystyle{abbrvnat}
\bibliography{references}

\newpage

\section{Appendix A. Spectral bounds}

\begin{lemma}\label{spectral_bound}
Suppose $TK$ points $\{z_1,z_2, \dots z_{TK}\}$ are sampled i.i.d from $\cW$ according to a Borel measure $\varrho$. Let $\bfZ=T\mathbf{\Lambda}+ \tau \mathbf{Id}$ where $\mathbf{\Lambda}=\mathbb{E}_{w\sim \varrho}[\phi(w)\phi(w)^{\top}]$ and :

\begin{align*}
    \tildeZ=\frac{1}{K}\sum_{i=1}^{TK} \phi(z_i)\phi(z_i)^{\top}+ \tau\mathbf{Id}
\end{align*}

By choosing parameter $K=\lceil T/\gamma_T\rceil$ under the condition that $T>\overline{N}_1(\delta)=\left(\frac{\log(1/\delta)}{C_pF^216}\right)^{2\beta_p/(\beta_p-1)}$ we claim with probability at least $1-\delta$:

\begin{align*}
    \|\tildeZ^{-1}\bfZ-\mathbf{Id}\|_2\leq  \frac{28}{17} \cdot \sqrt{\frac{\gamma_T}{T}\frac{\log(1/\delta)}{\tau}}
\end{align*}

\end{lemma}

\begin{proof}

 Note that we will first bound the spectral norm of $\|\bfZ^{-1}\tildeZ-\mathbf{Id}\|$. Consider a  $g \in \cH_k$  with $\|g\| = 1$. We have:
    \begin{align*}
        g^{\top}(\bfZ^{-1}\tildeZ  - \Id) g = \sum_{j = 1}^{KT} \frac{1}{K} g^{\top} \bfZ^{-1} \phi(x_i) \phi(x_i)^{\top} g - g^{\top}\bfZ^{-1} (T\mathbf{\Lambda}) g.
    \end{align*}
    
   Define $V_i := \frac{1}{K} g^{\top} \bfZ^{-1} \phi(z_i) \phi(z_i)^{\top} g$.  We can now write:
    \begin{align*}
        \E[V_i] & = \frac{1}{K} g^{\top} \bfZ^{-1} \mathbf{\Lambda} g \\
        |V_i| &= \sup_{z,z_i\in \cW}\sup_{g \in \cH_k}   \frac{1}{K} g^{\top} \bfZ^{-1} \phi(z_i) \phi(z_i)^{\top} g\leq \sup_{z,z_i\in \cW}\sup_{g \in \cH_k}\frac{g(z)}{K}\sqrt{g^{\top}\bfZ^{-1}g}\sqrt{\phi(z_i)^\top\bfZ^{-1}\phi(z_i)}\leq \\
        &\leq\frac{1}{K} \max\{ \sup_{g} g^{\top} \bfZ^{-1} g, \sup_{z} \phi^{\top}(z) \bfZ^{-1} \phi(z) \} =C_1\\
        \E[V_i^2] &\leq \frac{1}{K^2}\E\left[\left(g^{\top} \bfZ^{-1} \phi(z_i) \phi(z_i)^{\top} g\right)^2\right]= \frac{1}{K^2}\E\left[g(z)^2 g^{\top}\bfZ^{-1}\phi(z_i)\phi(z_i)^{\top}\bfZ^{-1} g\right]\leq \frac{1}{K^2} g^{\top} \bfZ^{-1} \mathbf{\Lambda} \bfZ^{-1} g \implies\\
        \sum_{i=1}^{KT} \E[V^2_i]&\leq \frac{1}{K}g^{T}\bfZ^{-1}T\mathbf{\Lambda}\bfZ^{-1}g\leq \frac{1}{K}g^{\top}\bfZ^{-1}g=C_0
    \end{align*}

    Note that in deriving the second and the third inequality we used $\sup_z g(z)\leq 1$, which follows as $g(z)=\langle g, \phi(z) \rangle_{\cH_k}\leq\|g\|_{\cH_k}\|\phi(z)\|_{\cH_k}=1$
    
    Applying Bernstien inequality to the collection of random variables $\{V_i\}_{i=1}^{KT}$ we obtain:

    \begin{align*}
    P\left(\left|\sum_{i=1}^{KT} V_i-g^{\top}\bfZ^ {-1}T\mathbf{\Lambda} g\right|>r\right)\leq 2\exp\left(-\frac{r^2}{2\left(C_0+C_1r/3\right)}\right)
    \end{align*}

    Taking  $r=\sqrt{\frac{1}{K} \cdot g^{\top} \bfZ^{-1} g \cdot \log(1/\delta)} + \frac{2\log(1/\delta)}{3K} \max\{ \sup_{g} g^{\top} \bfZ^{-1} g, \sup_{z} \phi^{\top}(z) \bfZ^{-1} \phi(z) \}$ we have with probability at least $1-\delta$:
    \begin{align*}
        \|\bfZ^{-1}\tildeZ  - \Id\|_2 \leq \sqrt{\frac{1}{K} \cdot g^{\top} \bfZ^{-1} g \cdot \log(1/\delta)} + \frac{2\log(1/\delta)}{3K} \max\{ \sup_{g} g^{\top} \bfZ^{-1} g, \sup_{z} \phi^{\top}(z) \bfZ^{-1} \phi(z) \}.
    \end{align*}

    Substituting $K=\lceil\frac{T}{\gamma_T}\rceil$ and noting that $\bfZ^{-1}\prec \tau^{-1} \mathbf{Id}, K\geq T/\gamma_T$ we obtain:

     \begin{align*}
        \|\bfZ^{-1}\tildeZ  - \Id\|_2 \leq \sqrt{\frac{\gamma_T}{T}\frac{1}{\tau}\log(1/\delta)}+\frac{\gamma_T}{T}\frac{2}{3\tau}\log(1/\delta)
    \end{align*}

    To obtain a bound on the spectral norm of $\tildeZ^{-1}\bfZ-\mathbf{Id}$ note that if $\|\bfZ^{-1}\tildeZ  - \Id\|_2\leq b$  the eigen- values of   $\bfZ^{-1}\tildeZ$ belong to the interval $(1-b,1+b)$. Thus the eigen-values of $\tildeZ^{-1}\bfZ$ are contained in an interval $((1+b)^{-1},(1-b)^{-1})$. We can finally conclude with probability at least $1-\delta$:

    \begin{align*}
        \|\tildeZ^{-1}\bfZ-\mathbf{Id}\|_2\leq \frac{b}{1-b}\leq \left(\sqrt{\frac{\gamma_T}{T}\frac{1}{\tau}\log(1/\delta)}+\frac{\gamma_T}{T}\frac{2}{3\tau}\log(1/\delta)\right)\frac{1}{1-\left(\sqrt{\gamma_T/T\tau\log(1/\delta)}+2\gamma_T/3T\tau\log(1/\delta)\right)}
    \end{align*}

    Along with the condition that $T>(\frac{\log(1/\delta)}{C_pF^216})^{2\beta_p/(\beta_p-1)}=\overline{N}_1(\delta)$ which implies \citep{information_gain_bound} $\sqrt{\frac{\gamma_T}{T}\frac{1}{\tau}\log(1/\delta)}\leq 1/4$, we can simplify the expression as:

    \begin{align*}
          \|\tildeZ^{-1}\bfZ-\mathbf{Id}\|_2\leq \frac{28}{17} \cdot \sqrt{\frac{\gamma_T}{T}\frac{\log(1/\delta)}{\tau}}
    \end{align*}

\end{proof}

\begin{lemma}\label{approximation_lemma}
Consider an $T$ i.i.d point $\{w_i\}_{i=1}^{T}$ sampled according to a Borel measure $\varrho$ from the domain $\cX$ and introduce an estimator $\mu_T(w)=\phi(w)^{\top}\hatZ^{-1}\Phi_Ty_T$, where $\hatZ=\Phi_T\Phi_T^{\top}+\tau\mathbf{Id}$. Let $\{z_i\}^{KT}_{i=1}$ be another set of i.i.d sampled point 
according to $\varrho$ and let $\tildeZ=\frac{1}{K}\sum_{i=1}^{KT}\phi(z_i)\phi(z_i)^{\top}+\tau\mathbf{Id}$ be the same operator introduced in Lemma \ref{spectral_bound}. Define a new estimator:
\begin{align*}      \overline{\mu}_T(w)=\phi(w)^{\top}\tildeZ^{-1}\Phi_Ty_T
\end{align*}
Fix an arbitrary $w\in \cW$, then with probability at least $1-\delta$ :
\begin{align*}
    |\mu_T(w)-\overline{\mu}_T(w)|\leq \beta_1(\delta)\sqrt{\frac{\gamma_T}{T}}
\end{align*}
where $\beta_1(\delta)=\left(\sqrt{162B^3/13F^2}+81F^2/52\right)\log(8/\delta)+  28/17\frac{\log(4/\delta)}{\tau}+2B\sqrt{108F^2\tau/13}+4R\log(4/\delta)\sqrt{243F^2/26}$

\end{lemma}

\begin{proof}

    We can re-write the estimator difference as:
    \begin{align}\label{approximation_lemma:eq_1}
        |\mu_T(w)-\overline{\mu}_T(w)|&=|\phi(x)^{\top}(\tildeZ^{-1}-\bfZ^{-1})\Phi_Ty_T|\leq\nonumber \\
        &\leq |\phi(w)^{\top}\hatZ^{-1}\Phi_T\varepsilon_{1:T}|+|\phi(w)^{\top}\tildeZ^{-1}\Phi_T\varepsilon_{1:T}|+|\phi(w)^{\top}(\tildeZ^{-1}- \hatZ^{-1})\Phi_T\Phi_T^{\top}f|
    \end{align}
    We start by bounding the first two terms of eq.(\ref{approximation_lemma:eq_1}). We will first bound the norms  of the vectors $\|\phi(w)^{\top}\hatZ^{-1}\Phi_T\|,\|\phi(w)^{\top}\tildeZ^{-1}\Phi_T\|$ and then utilize the fact $\varepsilon_{1:T}$ is an R-sub-Gaussian independent from both vectors.
    \begin{align*}
        \|\phi(w)^{\top}\hatZ^{-1}\Phi_T\|^2_2 &\leq \sup_{w\in \cW} \phi^{\top}(w) \hatZ^{-1} \Phi_ T \Phi^{\top}_T \hatZ^{-1} \phi(w)\leq\\
        &\leq \sup_{w\in \cW} \phi^{\top}(w)  \hatZ^{-1} \phi(w)
    \end{align*}
    Where, in the second line we use the identity $ \Phi_T\Phi_T^{\top}=\hatZ-\tau\mathbf{Id}\prec\hatZ$. Next not that $\varepsilon_{1:T}$ is as an R-sub-Gaussian vector and thus after using \cite[Lemma 3.3,Lemma 3.4]{Sudeep_Uniform_Sampling} we can bound the dot product with probability at least $1-\delta$:
\begin{align}\label{approximation_lemma:eq_8}|\phi(w)^{\top}\hatZ^{-1}\Phi_T\varepsilon_{1:T}|\leq 2R\log(2/\delta)\sqrt{\sup_{w} \phi^{\top}(w)  \hatZ^{-1} \phi(w)}\leq 2R\log(1/\delta)\sqrt{108F^2/13} \sqrt{\frac{\gamma_T}{T}}
    \end{align}

    We continue by bounding the norm $\|\phi(w)^{\top}\tildeZ^{-1}\Phi_T\|_2$. We will utilize the result of \cite[Lemma 3.2]{Sudeep_Uniform_Sampling} that allows to say for $T>\overline{N}(\delta/4)$ w.p $1-\delta/4$ $\|\bfZ^{-0.5}\hatZ\bfZ^{-0.5}-\mathbf{Id}\|_2\leq 1/9$. Note that $\bfZ^{-0.5}\hatZ\bfZ^{-0.5},\hatZ\bfZ^{-1}$ have the same spectrum's and hence $8/9\mathbf{Id}\prec\hatZ\bfZ^{-1}\prec10/9\mathbf{Id}$. As both $\bf{Z},\hatZ$ are positivite-definite matrices we can conclude $8/9\bfZ\prec\hatZ\prec10/9\bfZ$. Using this result we can write:

    \begin{align*}
        \|\phi(w)^{\top}\tildeZ^{-1}\Phi_T\|^2_2&\leq \sup_{w\in \cW}\phi(w)\tildeZ^{-1}\Phi_T\Phi_T^{\top}\tildeZ^{-1}\phi(w)\leq\sup_{w\in \cW} \phi^{\top}(w) \tildeZ^{-1} \Phi \Phi^{\top} \tildeZ^{-1} \phi(w)  \\
        & \leq \sup_{w\in \cW} \phi^{\top}(w) \tildeZ^{-1} \hatZ \tildeZ^{-1} \phi(w)  \\
        & \leq 10/9 \cdot\sup_{w} \phi^{\top}(w) \tildeZ^{-1} \bfZ \tildeZ^{-1} \phi(w)  \\
        & \leq 10/9 \cdot \sup_{w} \phi^{\top} (w) \tildeZ^{-1} \phi(w)  \|\tildeZ^{-1/2} \bfZ \tildeZ^{-1/2}\|_2
    \end{align*}

    To bound $\|\tildeZ^{-1/2} \bfZ \tildeZ^{-1/2}\|_2$ we will use Lemma \ref{spectral_bound} that provides the bound $\|\tildeZ^{-1}\bfZ\|_2<3/2$, w.p $1-\delta/4$ for $T\geq \overline{N_1}(\delta/4)$. Note that $\tildeZ^{-1/2} \bfZ \tildeZ^{-1/2}, \tildeZ^{-1}\bfZ$ have the same spectrum's and thus $\|\tildeZ^{-1/2}\bfZ\tildeZ^{-1/2}\|_2\leq 3/2$. We will once again use Lemma \ref{spectral_bound} in bounding  $\sup_{w} \phi^{\top} (w) \tildeZ^{-1} \phi(w)$. Namely note that $\|\tildeZ^{-1}\bf{Z}-\mathbf{Id}\|_2\leq \frac{1}{2}$ and thus $1/2\prec\tildeZ^{-1}\bfZ\prec 3/2$ . As both $\bfZ,\tildeZ$ are positive definite we can conclude $1/2\tildeZ\prec\bfZ\prec3/2\tildeZ \implies 2/3\tildeZ^{-1}\prec \bfZ^{-1}\prec 2\tildeZ^{-1}$ and hence:

    \begin{align*}
     \sup_w\phi(w)^{\top}\tildeZ^{-1}\phi(w)\leq 3/2\sup\phi(w)^{\top}\bfZ^{-1}\phi(w)\leq \frac{\gamma_T}{T}81F^2/13
    \end{align*}

    Thus finally:

    \[
    |\phi(w)^{\top}\tildeZ^{-1}\Phi_T\|^2_2\leq 243F^2/26 \frac{\gamma_T}{T}
    \]

    Repeating the previous argument for the R-sub-Gaussian vector $\varepsilon_{1:T}$ we can write w.p at least $1-\delta$:
    \begin{align}\label{approximation_lemma:eq_7}
    |\phi(w)^{\top}\tildeZ^{-1}\Phi_T\varepsilon_{1:T}|\leq 2R\log(1/\delta)\sqrt{243F^2/26}\sqrt{\frac{\gamma_T}{T}}
    \end{align}

    We now turn our attention to the first term of eq.(\ref{approximation_lemma:eq_1}):

    \begin{align}\label{approximation_lemma:eq_2}
        |\phi^{\top}(w) \tildeZ^{-1} \Phi \Phi^{\top} f - \phi^{\top}(w) \hatZ^{-1} \Phi \Phi^{\top} f| & = |\phi^{\top}(w) (\tildeZ^{-1} - \hatZ^{-1}) (\hatZ - \tau \Id) | \leq\nonumber \\
        & \leq |\phi^{\top}(w) ( \tildeZ^{-1} \hatZ-\mathbf{Id})f  | +\tau |\phi^{\top}(w) \hatZ^{-1} f| + \tau| \phi^{\top}(w) \tildeZ^{-1} f|
    \end{align}

    We will first  bound the last two terms of eq.(\ref{approximation_lemma:eq_2}):

    \begin{align}\label{approximation_lemma:eq_5}
        \tau|\phi^{\top}(w) \hatZ^{-1} f|&\leq \tau\sqrt{\phi^{\top}(w) \hatZ^{-1}\phi(w)}\sqrt{f^{\top}\hatZ^{-1}f}\leq\nonumber\\
        &\leq B\sqrt{108F^2\tau/13}\sqrt{\frac{\gamma_T}{T}}
    \end{align}

    The first inequality is Cauchy-Schwartz while the second follows from \cite[Lemma 3.3,Lemma 3.4]{Sudeep_Uniform_Sampling} and $\hatZ^{-1}\prec\tau^{-1}\mathbf{Id}$. We follow the same methodology in bounding the second term:

    \begin{align}\label{approximation_lemma:eq_6}
        \tau| \phi^{\top}(w) \tildeZ^{-1} f|&\leq \tau\sqrt{\phi^{\top}(w) \tildeZ^{-1}\phi(w)}\sqrt{f^{\top}\tildeZ^{-1}f}\leq\nonumber \\
        &\leq B\sqrt{81F^2\tau/13}\sqrt{\frac{\gamma_T}{T}}
    \end{align}

    Here we once again used the previously derived inequality, $3/2 
    \bfZ^{-1}\succ \hatZ^{-1}$.\\

    To  bound the first term of eq.(\ref{approximation_lemma:eq_2}), $|\phi^{\top}(w) ( \tildeZ^{-1} \hatZ-\mathbf{Id})f |$ we will utilize Bernstein inequality. Note that $\tildeZ,\hatZ$ are independent and hence by conditioning on $\tildeZ$ we may assume it is a fixed matrix. We can re-write $\phi^{\top}(w) ( \tildeZ^{-1} \hatZ-\mathbf{Id})f$ as:
    \begin{align*}
    \phi^{\top}(w) ( \tildeZ^{-1} \hatZ-\mathbf{Id})f&=\sum_{i=1}^{T}\phi(w)^{\top}\tildeZ^{-1}\phi(w_i)\phi(w_i)^{\top}f+\tau \phi(w)^{\top}\tildeZ^{-1}f-\phi(w)^{\top}f=\\
    &=\sum_{i=1}^{T}\phi(w)^{\top}\tildeZ^{-1}\phi(w_i)\phi(w_i)^{\top}f-\phi(w)^{\top}\tildeZ^{-1}\bfA f
    \end{align*}

    Where for notational convenience we introduce the notation $\bf{A}=\tildeZ-\tau \mathbf{Id}$ . Let $Q_i=\phi(w)^{\top}\tildeZ^{-1}\phi(w_i)\phi(w_i)^{\top}f$. We first bound the moments of $\{Q_i\}_{i=1}^{KT}$:
        \begin{align}\label{approximation_lemma:eq_3}
        \mathbb{E}[Q_i]&=\phi(w)^{\top}\tildeZ^{-1}\mathbf{\Lambda} f\\
        |Q_i|&=\sup_y \phi(y)^{\top}\tildeZ^{-1}\phi(w_i)\phi(w_i)^{\top}f\leq \sup_y\sqrt{\phi(y)^{\top}\tildeZ^{-1}\phi(y)^{\top}}\sqrt{\phi(y)^{\top}\tildeZ^{-1}\phi(y)}B\leq \nonumber\\
        &\leq B\sup_y\phi(y)^{\top}\tildeZ^{-1}\phi(y)^{\top}:=V_0\\
        \mathbb{E}[Q_i^2]&\leq \mathbb{E}[(\phi(w)^{\top}\tildeZ^{-1}\phi(w_i))^2B^2]=B^2\phi(w)^{\top}\tildeZ^{-1}\mathbf{\Lambda}\tildeZ^{-1}\phi(w)\implies \nonumber\\
        \implies& \sum_{i=1}^{T}\mathbb{E}[Q_i^2]=B^2\phi(w)^{\top}\tildeZ^{-1}T\mathbf{\Lambda}\tildeZ^{-1}\phi(w)\leq 3/2B^2\phi(w)^{\top}\tildeZ^{-1}\phi(w)\leq BV_0\label{approximation_lemma:eq_4}
    \end{align}

    In the last step we again used the  inequality $3/2\bfZ^{-1}\succ \tildeZ^{-1}$


    Applying Bernstein inequality with the moment bounds from (\ref{approximation_lemma:eq_3}-\ref{approximation_lemma:eq_4}) we have:
    \begin{align*}          P\left(\left|\sum_{i=1}^{T}\phi(w)^{\top}\tildeZ^{-1}\phi(w_i)\phi(w_i)^{\top}f-\phi(w)^{\top}\tildeZ^{-1}T\mathbf{\Lambda} f\right|\geq r\right)\leq 2\exp\left(-\frac{r^2}{2(BV_0+rV_0/3)}\right)
    \end{align*}

    Choosing  $r=(\sqrt{2BV_0}+2V_0/3)\log(2/\delta)$ ensures that w.p $1-\delta$:
    \begin{align*}
    \left|\sum_{i=1}^{T}\phi(w)^{\top}\tildeZ^{-1}\phi(w_i)\phi(w_i)^{\top}f-\phi(w)^{\top}\tildeZ^{-1}T\mathbf{\Lambda} f\right|\leq(\sqrt{2BV_0}+2V_0/3)\log(2/\delta)
    \end{align*}

    We can now bound the original expression as:
    \begin{align}\label{approximation_lemma:eq_9}
        &\left| \phi^{\top}(w) ( \tildeZ^{-1} \hatZ-\mathbf{Id})f\right|=\left|\sum_{i=1}^{T}\phi(w)^{\top}\tildeZ^{-1}\phi(w_i)\phi(w_i)^{\top}f-\phi(w)^{\top}\tildeZ^{-1}\bfA f\right|\leq\nonumber\\
        &\leq \left|\sum_{i=1}^{T}\phi(w)^{\top}\tildeZ^{-1}\phi(w_i)\phi(w_i)^{\top}f-\phi(w)^{\top}\tildeZ^{-1}T\mathbf{\Lambda} f\right|+\left|\phi(w)^{\top}\hatZ^{-1}(\bfA-T\mathbf{\Lambda})f\right|\leq \nonumber\\
        &\leq(\sqrt{2BV_0}+2V_0/3)\log(2/\delta)+\left|\phi(w)^{\top}\tildeZ^{-1}(\tildeZ-\bfZ)f\right|\leq\nonumber\\
        &\leq \sqrt{\frac{\gamma_T}{T}}\left(\sqrt{162B^3/13F^2}+ 81F^2/13 \sqrt{\frac{\gamma_T}{T}}\right)\log(2/\delta)+  28/17\sqrt{\frac{\gamma_T}{T}\frac{\log(1/\delta)}{\tau}}
    \end{align}

    In the last line we used the previously introduced inequality $\sup_{w}\phi(w)^{\top}\tildeZ^{-1}\phi(w)^{\top}\leq 81F^2/13 \frac{\gamma_T}{T}$ and the result of Lemma \ref{spectral_bound}.\\

    Finally adding all inequalities from eqs.(\ref{approximation_lemma:eq_9},\ref{approximation_lemma:eq_8},\ref{approximation_lemma:eq_7},\ref{approximation_lemma:eq_6},\ref{approximation_lemma:eq_5}) we obtain with probability $1-\delta$:
    \begin{align*}
    &|\mu_T(w)-\overline{\mu}_T(w)|\leq\\ 
    &\leq \sqrt{\frac{\gamma_T}{T}}\left(\left(\sqrt{162B^3/13F^2}+81F^2/52\right)\log(8/\delta)+  28/17\sqrt{\frac{\log(4/\delta)}{\tau}}+2B\sqrt{108F^2\tau/13}+4R\log(4/\delta)\sqrt{243F^2/26}\right)
    \end{align*}

    For notational convinience we introduce the shorthand notation: $$\beta_1(\delta)=\left(\left(\sqrt{162B^3/13F^2}+81F^2/52\right)\log(8/\delta)+ 28/17\sqrt{\frac{\log(4/\delta)}{\tau}}+2B\sqrt{108F^2\tau/13}+4R\log(4/\delta)\sqrt{243F^2/26}\right)$$

    \end{proof}

     We next prove that the estimator $\overline{\mu}_T$ approximates the reward function $f$ sufficiently well over the entire domain $\cX$. To this end we use the assumption \ref{assumption:Grid} condition to construct a grid and apply the confidence bounds on each point in the gid.
     
    \begin{lemma}\label{lemma:confidence_bounds}
    For the estimator $\overline{\mu}_T$ introduced in Lemma \ref{approximation_lemma} under the  condition that $T>\max(\overline{N}(\delta/4),\overline{N}_1(\delta/4))$  we claim with probability at least $1-\delta$:
    \begin{align*}
        \sup_{w\in\cW}|\overline{\mu}_T(w)-f(w)|\leq\frac{11/3B+3R\sqrt{\gamma_T81F^2/13\log(2/\delta)}}{T}+\left(\beta(\delta/4|\cU_T|)+\beta_1(\delta/4|\cU_T|)\right)\sqrt{\frac{\gamma_ T}{T}}
     \end{align*}

     Where $\beta(\delta)=\frac{2R}{\tau}\log\left(\frac{1}{\delta}\right)\sqrt{\frac{108F^2}{13}}$

    \end{lemma}

    \begin{proof}
        We use the standard discretization argument from \citep{Vakili_Kernel_Simple_Regret}. Let $\cU_T$ be the discretization described in the Assumption\ref{assumption:Grid}. By the argument following used in Lemma \ref{approximation_lemma} we have $\tildeZ^{-1}\prec 3/2\bfZ^{-1}\prec 5/3\hatZ^{-1}$ and hence :
        \begin{align*}
            &\|\overline{\mu}_T\|_{\cH_k}\leq\|\tildeZ^{-1}\Phi_T\Phi_T^{\top}f\|_{\cH_k}+\|\tildeZ^{-1}\Phi_T\varepsilon_T\|_{\cH_K} \leq \\
            &\leq 5/3B+\sup_{g \in \cH_K} g^{\top}\tildeZ^{-1}\Phi_T\varepsilon_{1:T}=5/3B+\sup_{g \in \cH_k}\sum_{i=1}^{T}g^{\top}\tildeZ^{-1}\phi(w_i)\varepsilon_i\leq\\
            &\leq 5/3B+\sup \sum_{i=1}^{T}\sqrt{g^{\top}\tildeZ^{-1}g}
            \sqrt{\phi(w_i)^{\top}\tildeZ^{-1}\phi(w_i)}\varepsilon_i\leq \\
            &\leq 5/3B+\sqrt{1/\tau}\sum_{i=1}^{T}\sqrt{\phi(w_i)^{\top}\tildeZ^{-1}\phi(w_i)}\varepsilon_i
        \end{align*}
        Once again $\tildeZ^{-1}\prec 3/2\bfZ^{-1}$ and thus w.p at least $1-\delta$: $\phi(w_i)^{\top}\tildeZ^{-1}\phi(w_i)\leq 81F^2/13\gamma_T/T$. $\{\varepsilon_i\}_{i=1}^{T}$ are $R$-sub-Gaussian random variables and hence $\sum_{i=1}^{T}\sqrt{\phi(w_i)^{\top}\tildeZ^{-1}\phi(w_i)}\varepsilon_i$ is $R\sqrt{81F^2/13\gamma_T}$-sub-Gaussian. We now finally have with probability at least $1-\delta/4$:
        \begin{align*}
            &\|\overline{\mu}_T\|_{\cH_k}\leq 5/3B+2R\sqrt{(81F^2/13\tau)\gamma_T}\log(4/\delta)
        \end{align*}
        Repeating the same argument for $\hatZ$ we have with probability at least $1-\delta/4$
        \begin{align*}
            &\|\mu_T\|_{\cH_k}\leq B+2R\sqrt{(81F^2/13\tau)\gamma_T}\log(4/\delta)
        \end{align*}
    
         Recall that $[w]_T=\argmin_{y\in \cU_T}\|w-y\|_2$. We can now write w.p $1-3\delta/4$:
        \begin{align*}
            &\forall w\in \cW, |\overline{\mu}_T(w)-\mu_T(w)|\leq \\
            &\leq|\overline{\mu}_T([w]_T)-\overline{\mu}_T(w)|+|\mu_T([w]_T)-\mu_T(w)|+|\overline{\mu}_T([w]_T)-\mu_T([w]_T)|\leq \\
            & \leq \frac{8/3B+\sqrt{\gamma_T}2R\sqrt{81F^2/13\log(2/\delta)}}{T}+\beta_1(\delta/4|\cU_T|)\sqrt{\frac{\gamma_T}{T}}
        \end{align*}
     The last line follows by applying the  Lemma \ref{approximation_lemma} over the entire discrete grid $\cU_T$ and utilizing Assumption(\ref{assumption:Grid}).\\

     By using \cite[Theorem 1.]{Vakili_Kernel_Simple_Regret} and noting that by \cite[Lemma 3.3, 3.4]{Sudeep_Uniform_Sampling} $\sup_{w\in \cW} \tau\phi(w)^{\top}\hatZ^{-1}\phi(w)=\sigma^2_T(w)\leq 108F^2/13 \frac{\gamma_T}{T}$ we can now write w.p $1-\delta$:
     \begin{align*}
         &\forall w\in \cW, |f(w)-\overline{\mu}_T(w)|\leq\\
         &\leq |f(w)-\mu_T(w)|+|\mu_T(w)-\overline{\mu}_T(w)|\leq \\
        &\leq |f([w]_T)-\mu_T([w]_T)|+|f([w]_T)-f(w_T)|+|\mu_T([w]_T)-\mu_T(w_T)|+|\mu_T(w)-\overline{\mu}_T(w)|\leq\\
        &\leq \frac{11/3B+3R\sqrt{\gamma_T81F^2/13\log(2/\delta)}}{T}+\left(\beta(\delta/4|\cU_T|)+\beta_1(\delta/4|\cU_T|)\right)\sqrt{\frac{\gamma_ T}{T}}
     \end{align*}

     Where $\beta(\delta)=\frac{2R}{\tau}\log\left(\frac{1}{\delta}\right)\sqrt{\frac{108F^2}{13}}$
    \end{proof}

    \newpage

\section{Appendix B: Privacy Constraints}

Here we proved our algorithm achives $\varepsilon$-JDP privacy. First we bound the sensitivity of the estimator $\overline{\mu}_T$:

\begin{lemma}\label{Appendix :Lemma:sensitivity_bound}
Let $\cS_T,\cS^{'}_T$ be two $t$-neighbouring databases, and let $\overline{\mu}_T,\overline{\mu}'_T$ be the 2 estimator constructed for each of the databases . For $T>\max(\overline{N}_1(\delta/4),\overline{N}(\delta/4))$, we can bound the sensitivity of an estimator $\overline{\mu}_T$ in Algortihm(\ref{alg:Algorithm 3}) as :

\begin{align*}
    \Delta\overline{\mu}_T=\sup_{w\in \cW}\sup_{\cS_T,\cS_T^{'}\text{are t neigbours}}|\overline{\mu}_T(w)-\overline{\mu}'_T(w)|\leq 2B\sup_{w\in \cW} \overline{\sigma}^2(w)
\end{align*}
 Furthermore we can bound  the posterior  variance w.p. $1-\delta$ as:

\begin{align*}
    \sup_{w\in \cW} \overline{\sigma}^2(w)\leq \frac{81F^2\gamma_T}{13T}
\end{align*}

\end{lemma}

\begin{proof}
Recall the parametric form of $\overline{\mu}_T$ introduced in the Lemma \ref{approximation_lemma}
\begin{align*}
\overline{\mu}_T(w)=\phi(w)^{\top}\tildeZ^{-1}\Phi_{\bfW_T}\bfY_T
\end{align*}

Where $\tildeZ=1/\lceil T/\gamma_T\rceil\Phi_{\cZ}\Phi_{\cZ}^{\top}+\tau\mathbf{Id}$. Note that in \textsc{USCA} query points $\{x_i\}_{i=1}^{T}$ are chosen only depending on the domain $\cX$  and are independent from previous rewards and contexts. We can thus conclude that the dataset, $\cD_T(\cS_T)=\{w_1, w_2\dots w_T\}, w_i=(x_i,c_i)$, consisting of (point, context) pairs only differs at time $t$ i.e $\cD_T(\cS_T)\triangle \cD_T(\cS_T')=\{(x_t,c_t),(x_t,c_t')\} $. Introduce shorthand notation $w=(c,x)$ we can now write:
\begin{align*}
    |\overline{\mu}_T(w)-\overline{\mu}'_T(w)|&=\left|\sum_{w_i\in \cD_T(\cS_T)}\phi(w)^{\top}\tildeZ^{-1}\phi(w_i)y_i-\sum_{(w_i)\in \cD_T(\cS_T')}\phi(w)^{\top}\tildeZ^{-1}\phi(w_i)y_i\right|\leq\\
    &\leq \left|\phi(w)^{\top}\tildeZ^{-1}\phi(w_t)y_t\right|+\left|\phi(w)^{\top}\tildeZ^{-1}\phi(w'_t)y'_t\right|\leq\\
    &\leq 2B\sup_{w\in \cW}\phi(w)^{\top}\tildeZ^{-1}\phi(w)
\end{align*}

Here the last line stems from the assumption that reward are bounded by $B$(see assumption(\ref{assumption: bounded_rewards})). By the result of Lemma \ref{spectral_bound}  for $T>\max(\overline{N}_1(\delta/2),\overline{N}(\delta/2))$ with probability at least $1-\delta/2$ $\|\tildeZ^{-1}\bfZ\|_2<3/2\implies 3/2\bfZ^{-1}\succ \tildeZ^{-1}$. Thus, by \cite[Lemma 3.4]{Sudeep_Uniform_Sampling} we have with probability at least $1-\delta$ , $\sup_{w\in \cW}\phi(w)^{\top}\tildeZ^{-1}\phi(w)<3/2\sup_{g\in \cW}\phi(w)^{\top}\bfZ^{-1}\phi(w)<81F^2/13 \frac{\gamma_T}{T}$ .Hence  with probability at least $1-\delta$:

\begin{align*}
    \sup_{w\in \cW} \overline{\sigma}^2(w)\leq \frac{81F^2\gamma_T}{13T}
\end{align*}

\end{proof}

We next show the output of $\widehat{X}_T(c_{T+1})$ is $\varepsilon$-DP with respect to the previously seen history. The proof closely follows the Theorem 6 of \citep{McSherryadnTalwar}.

\begin{lemma}
$\widehat{x}_{T}(c_{T+1})$ is $\varepsilon$ -DP with respect to the database $\cS_T\setminus \{c_{T+1}\}$
\end{lemma}
\begin{proof}
    As usual denote $\cS_T,\cS'_T$ to be two $t$-neighbouring databases. For the sake of space, introduce the shorthand notation $c_{T+1}\equiv c$. Recall that $\cZ$ is the aproximating set of the algorithm. We can now write:
    \begin{align*}
        \frac{P(\widehat{x}_{T}(c)=r)}{P(\widehat{x}'_{T}(c)=r)}=\frac{\int_{\cZ} \frac{\exp(\varepsilon \overline{\mu}_T(r,c)/(4B\sup\overline{\sigma}^2)}{\int_{\cX} \exp(\varepsilon \overline{\mu}_T(r,c)/(4B\sup\overline{\sigma}^2)\nu_0(dr) }d\cZ}{\int_{\cZ} \frac{\exp(\varepsilon \overline{\mu}'_T(r,c)/(4B\sup\overline{\sigma}^2)}{\int_{\cX} \exp(\varepsilon \overline{\mu}'_T(r,c)/(4B\sup\overline{\sigma}^2)\nu_0(dr) }d\cZ}
    \end{align*}

By using the definition of sensitivity we can bound the  ratio :

\begin{align*}
    \frac{ \frac{\exp(\varepsilon \overline{\mu}_T(r,c)/(4B\sup\overline{\sigma}^2)}{\int_{\cX} \exp(\varepsilon \overline{\mu}_T(r,c)/(4B\sup\overline{\sigma}^)\nu_0(dr) }}{\frac{\exp(\varepsilon \overline{\mu}'_T(r,c)/(4B\sup\overline{\sigma}^2)}{\int_{\cX} \exp(\varepsilon \overline{\mu}'_T(z,c)/(4B\sup\overline{\sigma}^2)\nu_0(dr) }}&\leq \exp(\varepsilon \Delta\overline{\mu}_T/(4B\sup\overline{\sigma}^2))\frac{\int_{\cX} \exp(\varepsilon \overline{\mu}'_T(r,c)/(4B\sup\overline{\sigma}^2)\nu_0(dr)}{\int_{\cX} \exp(\varepsilon \overline{\mu}_T(r,c)/(4B\sup\overline{\sigma}^2)\nu_0(dr)}\leq \\
    &\leq \exp(2\varepsilon\Delta\mu_T/(4B\sup\overline\sigma^2))
\end{align*}

Note that $\overline{\sigma}^2$ is only a function of $\cZ$ and does not depend on the database $\cS_T$. We can hence write:

\begin{align*}
        \frac{P(\widehat{x}_{T}(c)=r)}{P(\widehat{x}'_{T}(c)=r)}\leq \frac{\int_{\cZ} \exp(2\varepsilon\Delta \overline{\mu}_T/(4B\sup\overline{\sigma}^2)) \frac{\exp(\varepsilon \overline{\mu}'_T(r,c)/(4B\sup\overline{\sigma}^2)}{\int_{\cX} \exp(\varepsilon \overline{\mu}'_T(r,c)/(4B\sup\overline{\sigma}^2)\nu_0(dr) }d\cZ}{\int_{\cZ} \frac{\exp(\varepsilon \overline{\mu}'_T(r,c)/(4B\sup\overline{\sigma}^2)}{\int_{\cX} \exp(\varepsilon \overline{\mu}'_T(r,c)/(4B\sup\overline{\sigma}^2)\nu_0(dr) }d\cZ}
\end{align*}

Recall that lemma \ref{Appendix :Lemma:sensitivity_bound} ensures that $\Delta\overline{\mu}(T)\leq 2B\sup\overline{\sigma}^2$ with probability 1 over the randomness generated by $\cZ$. We can hence write:

\begin{align*}
    \frac{P(\widehat{x}_{T}(c)=r)}{P(\widehat{x}'_{T}(c)=r)}\leq \frac{\int_{\cZ} \exp(2\varepsilon\Delta\mu_T/(4B\sup\overline{\sigma}^2)) \frac{\exp(\varepsilon \overline{\mu}'_T(r,c)/(4B\sup\overline{\sigma}^2)}{\int_{\cX} \exp(\varepsilon \overline{\mu}'_T(r,c)/(4B\sup\overline{\sigma}^2)\nu_0(dr) }d\cZ}{\int_{\cZ} \frac{\exp(\varepsilon \overline{\mu}'_T(r,c)/(4B\sup\overline{\sigma}^2)}{\int_{\cX} \exp(\varepsilon \overline{\mu}'_T(r,c)/(4B\sup\overline{\sigma}^2)\nu_0(dr) }d\cZ}\leq \exp(\varepsilon)
\end{align*}

Which is what was originally claimed.

\end{proof}

\begin{theorem}\label{Appendix:JDP_guarantee}
    \textsc{USCA} is $\varepsilon$-JDP.
\end{theorem}

\begin{proof}

By previos lemma $\widehat{x}_T(c_{T+1})$ is $\varepsilon$-DP wrt the database $\cS_T\setminus c_{T+1}$. To now formally show that \textsc{USCA} satisfies the  $\varepsilon$-JDP, fix two $t$-neighbouring databases $\cS_T,\cS_T$, i.e $\cS_T\triangle\cS_T^{'}=\{(c_t,y_t),(c'_t,y'_t)\}$.\\
If $t=T+1$ there is nothing to show , as before time $T+1$ the points are chosen non-adaptively so their distribution does not depend on the contexts or rewards. Assume $t\leq T$, we have:

\begin{align}\label{privacy:eq_1}
    &\frac{P\left((x_{t+1},x_{t+2},\dots x_T, \widehat{x}_{T}(c_{T+1})=(r_{t+1},r_{t+2}, \dots r_T,r_{T+1})\right)}{P\left((x'_{t+1},x'_{t+2},\dots x'_T,\widehat{x}'_{T}(c_{T+1}))=(r_{t+1},r_{t+2}, \dots r_T, r_{T+1})\right)}=\frac{P(\widehat{x}_{T}(c_{T+1})=r_{T+1}|\cap^{T}_{j=t+1}(x_j=r_j,y_j,c_j))}{P(\widehat{x}_T^{'}(c_{T+1})=r'_{T+1}|\cap^{T}_{j=t+1}(x_j=r_j,y_j,c_j)))}
\end{align}
 The second equality follows as before time $T$ al points are queried uniformly from the domain $\cX$. They thus have the same distribution, independent of previous contexts and rewards.\\
To utilize the previously established result on $\varepsilon$-DP , we need to first fix the previous points $\{x_i\}_{i=1}^{t-1}$ . We do this by applying the law of total probability to the numerator and the denominator of eq.(\ref{privacy:eq_1}):
\begin{align*}
    &\frac{P(\widehat{x}_{T}(c_{T+1})=r_{T+1}|\cap^{T}_{j=t+1}(x_j=r_j,y_j,c_j))}{P(\widehat{x}'_{T}(c_{T+1})=r_{T+1}|\cap^{T}_{j=t+1}(x_j=r_j,y_j,c_j)))}=\\
    &=\frac{\int_{\cM_{t-1}}P(x_{T}(c_{T+1})=r_{T+1}|\cap^{T}_{j=1,j\not=t}(x_j=r_j,y_j,c_j))P(\cap_{j=1}^{t-1}(x_j=r_j,y_j,c_j))}{\int_{\cM_{t-1}} P(x'_{T}(c_{T+1})=r_{T+1}|\cap^{T}_{j=1,j\not=t}(x_j=r_j,y_j,c_j))P(\cap_{j=1}^{t-1}(x_j=r_j,y_j,c_j))}
\end{align*}

Where the randomness in the integration the first $t-1$ query points which is succinctly denoted as $\cM_{t-1}$. Note that $x_{T+1}$ only depends on $\cS_T$ through previous (point, reward, context) triples. We can hence use the previously derived $\varepsilon$-DP of $\widehat{x}_T(c_{T+1})$ with respect to $\cS_T$ to now write:

\begin{align*}
    &  \frac{P(\widehat{x}_{T}(c_{T+1})=r_{T+1}|\cap^{T}_{j=t+1}(x_j=r_j,y_j,c_j))}{P(\widehat{x}'_{T}(c_T+1)=r_{T+1}|\cap^{T}_{j=t+1}(x_j=r_j,y_j,c_j)))}=\\
    &=\frac{\int_{\cM_{t-1}} P(\widehat{x}_{T}(c_{T+1})=r_{T+1}|\cap^{T}_{j=1,j\not=t}(x_j=r_j,y_j,c_j))P(\cap^{t-1}_{j=1}(x_j=r_j,y_j,c_j))}{\int_{\cM_{t-1}} P(\widehat{x}'_{T}(c_{T+1})=r_{T+1}|\cap^{T}_{j=1,j\not=t}(x_j=r_j,y_j,c_j))P(\cap_{j=1}^{t-1}(x_j=r_j,y_j,c_j))}\leq\\
    \leq &\frac{\int_{\cM_{t-1}} \exp(\varepsilon)P(\widehat{x}'_{T}(c_{T+1})=r_{T+1}|\cap^{T}_{j=1,j\not=t}(x_j=r_j,y_j,c_j))P(\cap_{j=1}^{t-1}(x_j=r_j,y_j,c_j))}{\mathop{\int}_{\cM_{t-1}} P(\widehat{x}'_{T}(c_{T+1})=r_{T+1}|\cap^{T}_{j=1,j\not=t}(x_j=r_j,y_j,c_j))P(\cap_{j=1}^{t-1}(x_j=r_j,y_j,c_j))} =\exp(\varepsilon)
\end{align*}

Plugging this result back into eq.(\ref{privacy:eq_1}) we finally have:
\begin{align*}
    &P\left((x_{t+1},x_{t+2},\dots \widehat{x}_{T}(c_{T+1}))=(r_{t+1},r_{t+2}, \dots r_{T+1})\right)\leq \\
&\leq\exp(\varepsilon)P\left((x'_{t+1},x'_{t+2},\dots \widehat{x}'_{T}(c_{T+1}))=(r_{t+1},r_{t+2}, \dots r_{T+1})\right)
\end{align*}

For every two neighbouring data-bases which is what was desired.
\end{proof}

\section{Appendix C. Utility Analysis}

The following Lemma will be useful in applying the Exponential mechanism:

\begin{lemma}\label{appendix:geometric_lemma}
Let $\nu_0$ be the Lebesgue measure on $\mathbb{R}^{d}$. Consider a convex set $\mathcal{X}\subset \mathbb{R}^{d}$ with  diameter bounded by $\mathrm{diam}\cX\leq D_0$ and a $L_g$-Lipschitz function $g:\cX\rightarrow \mathbb{R}$. For an arbitrary $r>0$, we have:
\begin{align*}
    \frac{\nu_0g^{-1}\left([g(x^*)-r, g(x^*)]\right)}{\nu_0\mathcal{X}}\geq  \min(1,(r/D_0L_g)^{d})
\end{align*}

Where $x^{*}=\mathrm{argsup}_{x\in \mathcal{X}} g(x)$.    
\end{lemma}
\begin{proof}
Note that if $r>D_0L_g$ then $r>\sup_{x,y\in \cX} \|g(x)-g(y)\|\leq L_g\sup\|x-y\|_2 \leq L_gD_0$ and thus $g^{-1}[g(x^*)-r, g(x^*)]\equiv \cX$, giving the desired inequality.

Assume now $r<D_0L_g$. Consider the image of $\cX$ under the homothety centered at $x^*$, $\bfH:x\rightarrow x^*(1-\eta_0)+x\eta_0$ where $\eta_0=(r/D_0L_g)$. Denote by $\cY=\bf{H}(\cX)$, we will show that $\cY\subseteq g^{-1}[g(x^*)-r, g(x^*)]$ from which the desired inequality will follow, as $\nu_0(\cY)/\nu_0(\cX)=\eta_0^d$.\\
Note that by convexity, clearly $\cY \subset \cX$. Fix an arbitrary point $z\in \cY$. Note that $\mathbf{H}$ scales distances by $\eta_0$. Indeed :
\begin{align*}
    \|\bfH(x_1)-\bfH(x_2)\|_2=\|\eta_0(x_1-x_2)\|_2=\eta_0\|x_1-x_2\|_2
\end{align*}

It thus follows that $\mathrm{diam} \cY=\eta_0\mathrm{diam}\cX\leq \eta_0D_0$. Hence we can bound $\|x^{*}-z\|_2\leq \eta_0D$, as $x^{*} \in  \cY$. By Lipschitz condition we can further write:
\begin{align*}
    |g(x^{*})-g(z)|\leq L_g\eta_0D_0=r
\end{align*}

Thus clearly $z\in g^{-1}[g(x^*)-r, g(x^*)]$. Note that this holds for all $z\in \cY$ and thus we have the desired $\cY\subseteq g^{-1}[g(x^*)-r, g(x^*)]$.
\end{proof}

We can now present a result characterizing the simple regret performance of \textsc{USCA}.

\begin{theorem}\label{appendix:theorem}

Assume the kernel function satisfies the polynomial eigen-decay condition(\ref{assumptio:polynomial_decay}) for $\beta_p>1$. For $L_f$-Lipshitz reward function and $T> \max(\overline{N}(\delta/4),\overline{N}_1(\delta/4))$ where $\overline{N}(\delta),\overline{N_1}(\delta)$ are $\delta$-dependant constants  introduced in \citep{Sudeep_Uniform_Sampling} and Lemma \ref{spectral_bound} respectively. We can bound the average simple regret of the output points $\widehat{x}_{T}(c_{T+1})$ w.p $1-\delta$ as :
\begin{align*}
    &\mathop{\E}_{c_{T+1} \sim \kappa}\left[\sup_{x \in \cX} f(x,c_{T+1}) - f(\widehat{x}_T(c_{T+1}),c_{T+1})\right]\leq \\
    &\leq 10\left(\frac{11/3B+3R\sqrt{\gamma_T81F^2/13\log(6/\delta)}}{T}+\beta_2(\delta/12|\cU_T|)\sqrt{\frac{\gamma_ T}{T}}\right)+\frac{\gamma_T}{T}\frac{1}{\varepsilon}\frac{648BF^2}{13}\left(d\log\left(\frac{2592BF^2TL_fD_0}{13\gamma_T}\varepsilon\right)+\log(3/\delta)\right) 
\end{align*}
Here $\beta_2(\delta)=\beta_1(\delta)+\beta(\delta)$ where $\beta(\delta)=\frac{2R}{\tau}\sqrt{\frac{108F^2}{13}}\log\left(\frac{1}{\delta}\right)$ and $\beta_1$ is a $\delta$-dependant constant introduced in Lemma \ref{approximation_lemma}. In the above expression the randomness is over contexts, rewards and random coins of the algorithm

\end{theorem}

\begin{proof}
We will first show that the sample $\widehat{x}_{T}(c_{T+1})\sim\cE(\overline{\mu}(  \cdot, c_{T+1}), \varepsilon, 2B\sup_{w \in \cW}\overline{\sigma}^2(w))$ is close to $\sup_{x\in \cX}\overline{\mu}_T(c_{T+1},x)$. The full argument will then follow from the confidence bounds derived for the estimator $\overline{\mu}_T$ in Lemma \ref{lemma:confidence_bounds}.\\
\noindent Define $\cA_r=\{x \in \cX|\overline{\mu}_T(x,c_{T+1})\geq \sup_{x\in \cX}\overline{\mu}_T(x,c_{T+1})-r\}$. Consider 2 events in the sigma-algebra spanned by $\{\bfW_T,\bfY_T, \cZ\}$:
\begin{align*}
    \Gamma_1&=\left\{\sup_{w \in \cW}\overline{\sigma}^2(w)\leq \frac{81F^2\gamma_T}{13T}\right\}\\
    \Gamma_2&=\left\{ \sup_{w\in\cW}|\overline{\mu}_T(w)-f(w)|\leq  \frac{11/3B+3R\sqrt{\gamma_T81F^2/13\log(2/\delta)}}{T}+\left(\beta(\delta/4|\cU_T|)+\beta_1(\delta/4|\cU_T|)\right)\sqrt{\frac{\gamma_ T}{T}}
    \right\}
\end{align*}

By Lemma \ref{lemma:confidence_bounds} $P(\Gamma_2)\geq 1- \delta$ and by Lemma \ref{spectral_bound} $P(\Gamma_1)\geq 1-\delta$ by union bound we have $P(\Gamma_1,\Gamma_2)\geq 1-2\delta$. Note that:
\begin{align}\label{utility:eq_2}
    P(\widehat{x}_{T}(c_{T+1})\in \overline{\cA}_r)\leq P(\widehat{x}_{T}(c_{T+1})\in \overline{\cA}_r|\Gamma_{1},\Gamma_{2})+2\delta\cdot 1 
\end{align}
We thus only need to bound $P(\widehat{x}_{T}(c_{T+1})\in \overline{\cA}_r|\Gamma_{1},\Gamma_{2})$. In further writing we drop the conditioning notation in the interest of space.  We can now write:
\begin{align*}
    &P(\widehat{x}_{T}(c_{T+1})\in \overline{\cA}_r)\leq \frac{P( \widehat{x}_{T}(c_{T+1})\in\overline{\cA}_r)}{P(\widehat{x}_{T}(c_{T +1})\in\cA_{r/2})}=\\
    &=\frac{\int_{\bfW_T,\bfY_T,\cZ}\int_{x\in \overline{\cA}_r} \frac{\exp(\varepsilon\overline{\mu}_T(x,c_{T+1})/4B\sup\overline{\sigma}^2)}{\int_{x\in \cX}\exp(\varepsilon\overline{\mu}_T(x,c_{T+1})/4B\sup\overline{\sigma}^2)}}{\int_{\bfW_T,\bfY_T,\cZ}\int_{x\in\cA_{r/2}} \frac{\exp(\varepsilon\overline{\mu}_T(x,c_{T+1})/4B\sup\overline{\sigma}^2)}{\int_{x\in \cX}\exp(\varepsilon\overline{\mu}_T(x,c_{T+1})/4B\sup\overline{\sigma}^2)}}
\end{align*}

We can bound the ratio inside the outer integral as:
\begin{align*}
&\frac{\int_{x\in \overline{\cA}_r} \frac{\exp(\varepsilon\overline{\mu}_T(x,c_{T+1})/4B\sup\overline{\sigma}^2)}{\int_{x\in \cX}\exp(\varepsilon\overline{\mu}_T(c,c_{T+1})/4B\sup\overline{\sigma}^2)}
}{\int_{x\in \cA_{r/2}} \frac{\exp(\varepsilon\overline{\mu}_T(x,c_{T+1})/4B\sup\overline{\sigma}^2)}{\int_{x\in \cX}\exp(\varepsilon\overline{\mu}_T(x,c_{T+1})/4B\sup\overline{\sigma}^2)}
}=\frac{\int_{x\in \overline{\cA}_r} \exp(\varepsilon\overline{\mu}_T(x,c_{T+1})/4B\sup\overline{\sigma}^2)
}{\int_{x\in \cA_{r/2}} \exp(\varepsilon\overline{\mu}_T(x,c_{T+1})/4B\sup\overline{\sigma}^2}\leq \\
&\leq \frac{\exp(\varepsilon(\sup\overline{\mu}_T(x,c_{T+1})-r)/4B\sup\overline{\sigma}^2\nu_0(\overline{\cA}_r)}{\exp(\varepsilon(\sup\overline{\mu}_T(x,c_{T+1})-r/2)/4B\sup\overline{\sigma}^2\nu_0(\cA_{r/2})}\leq \exp(-\varepsilon r/(8B\sup\overline{\sigma}^2))\frac{\nu_0(\overline{\cA}_r)}{\nu_0(\cA_{r/2})}\leq \\
&\leq  \exp(-\varepsilon r/(8B\sup\overline{\sigma}^2))\frac{\nu_0(\cX)}{\nu_0(\cA_{r/2})}\leq
\end{align*}
Plugging this back into the previous equation we have:
\begin{align}\label{utility:eq:1}
     &P(\widehat{x}_{T}(c_{T+1})\in \overline{\cA}_r)\leq \frac{\int_{\bfW_T,\bfY_T,\cZ} \exp(-\varepsilon r/(8B\sup\overline{\sigma}^2))\frac{\nu_0(\cX)}{\nu_0(\cA_{r/2})}\int_{x\in \cA_{r/2}} \frac{\exp(\varepsilon\overline{\mu}_T(x,c_{T+1})/4B\sup\overline{\sigma}^2)}{\int_{x\in \cX}\exp(\varepsilon\overline{\mu}_T(x,c_{T+1})/4B\sup\overline{\sigma}^2)}}{\int_{\bfW_T,\bfY_T,\cZ}\int_{x\in \cA_{r/2}} \frac{\exp(\varepsilon\overline{\mu}_T(x,c_{T+1})/4B\sup\overline{\sigma}^2)}{\int_{x\in \cX}\exp(\varepsilon\overline{\mu}_T(x,c_{T+1})/4B\sup\overline{\sigma}^2)}}
\end{align}

We now bound the ratio $\frac{\nu_0(\cX)}{\nu_0(\cA_{r/2})}$ conditioned on $\Gamma_{1},\Gamma_2$. To this end we use Lemma \ref{appendix:geometric_lemma} along with confidence bounds in Lemma \ref{lemma:confidence_bounds}
, implied by the event $\Gamma_2$.  Consider the set:
$$
\cA'=\left\{x\in\cX|f(x,c_{T+1})\geq \sup_{x\in \cX} f(x,c_{T+1})-r/2+2\left(\frac{11/3B+3R\sqrt{\gamma_T81F^2/13\log(2/\delta)}}{T}+\beta_2(\delta/4|\cU_T|)\sqrt{\frac{\gamma_ T}{T}}\right)\right\}
$$

Fix an arbitrary $z\in \cA'$. By the definition of the event $\Gamma_2$ and the definition of $\cA'$ we can now write:
\begin{align*}
    &\sup_{x\in \cX}\overline{\mu}_T(x,c_{T+1})-\overline{\mu}_T(z,c_{T+1})\leq \\
    &\leq f(x^{*}_{T+1},c_{T+1})-f(z,c_{T+1})+2\left(\frac{11/3B+3R\sqrt{\gamma_T81F^2/13\log(2/\delta)}}{T}+\beta_2(\delta/4|\cU_T|)\sqrt{\frac{\gamma_ T}{T}}\right)=\\
    &=\sup_{x\in \cX}f(x,c_{T+1})-f(z,c_{T+1})+2\left(\frac{11/3B+3R\sqrt{\gamma_T81F^2/13\log(2/\delta)}}{T}+\beta_2(\delta/4|\cU_T|)\sqrt{\frac{\gamma_ T}{T}}\right)\leq r/2
\end{align*}

It thus follows that $\forall z
\in \cA'$ we have $z\in \cA_{r/2}$ and thus $ \cA^{'}\subseteq \cA_{r/2}$. We can now directly bound  $\nu_0(\cA_{r/2})/\nu_0(\cX)$ from Lemma \ref{appendix:geometric_lemma} :
\begin{align*}
    \frac{\nu_0(\cA_{r/2})}{\nu_0(\cX)}\geq\frac{\nu_0(\cA')}{\nu_0(\cX)}\geq\min\left(1,\left(r/2-2\left(\frac{11/3B+3R\sqrt{\gamma_T81F^2/13\log(2/\delta)}}{T}+\beta_2(\delta/4|\cU_T|)\sqrt{\frac{\gamma_ T}{T}}\right)\right)^d(1/D_0L_f)^{d}\right)
\end{align*}
From the event $\Gamma_1$ we also know $\sup_{w \in \cW}\overline{\sigma}^2(w)\leq \frac{81F^2\gamma_T}{13T}$. Plugging both of these bound in eq.(\ref{utility:eq:1})and integrating  out the $\{\bfW_T,\bfY_T,\cZ\}$ we have:
\begin{align*}
    &P(\widehat{x}_T(c_{T+1})\in \overline{\cA}_r|\Gamma_1,\Gamma_2)\leq \frac{\exp\left(-\frac{13rT}{648F^2\gamma_T}\right)}{\min\left(1,\left(r/2-2\left(\frac{11/3B+3R\sqrt{\gamma_T81F^2/13\log(2/\delta)}}{T}+\beta_2(\delta/4|\cU_T|)\sqrt{\frac{\gamma_ T}{T}}\right)\right)^d(1/(D_0L_f))^{d}\right)}
\end{align*}

By choosing:
$$r=8\left(\frac{11/3B+3R\sqrt{\gamma_T81F^2/13\log(2/\delta)}}{T}+\beta_2(\delta/2|\cU_T|)\sqrt{\frac{\gamma_ T}{T}}\right)+\frac{\gamma_T}{T}\frac{1}{\varepsilon}\frac{648BF^2}{13}\left(d\log\left(\frac{2592BF^2TL_fD_0}{13\gamma_T}\varepsilon\right)+\log(1/\delta)\right)$$
and using  eq.(\ref{utility:eq_2})
we have with probability at least $1-3\delta$:
\begin{align}\label{Theorem_regret:eq_2}
&\overline{\mu}_T(x_{T+1},c_{T+1})\geq \sup_{x\in \cX}\overline{\mu}_T(x,c_{T+1})-\\
&-8\left(\frac{11/3B+3R\sqrt{\gamma_T81F^2/13\log(2/\delta)}}{T}+\beta_2(\delta/4|\cU_T|)\sqrt{\frac{\gamma_ T}{T}}\right)-\frac{\gamma_T}{T}\frac{1}{\varepsilon}\frac{648BF^2}{13}\left(d\log\left(\frac{2592BF^2TL_fD_0}{13\gamma_T}\varepsilon\right)+\log(1/\delta)\right)\nonumber   
\end{align}
Recall the notation $x^*_{T+1}=\mathrm{argsup}_{x\in \cX}f(x,c_{T+1})$. To finish the proof we once again use  Lemma \ref{Lemma:sensitivity_bound} to guarantee w.p at least $1-\delta$
\begin{align*}
    &f(x^{*}_{T+1},c_{T+1})-f(x_{T+1},c_{T+1})=\\
    &=f(x^{*}_{T+1},c_{T+1})-\overline{\mu}_T(x^{*}_{T+1},c_{T+1})-(f(x_{T+1},c_{T+1})-\overline{\mu}_T(x_{T+1},c_{T+1}))+(\overline{\mu}_T(x^{*}_{T+1},c_{T+1})-\overline{\mu}_T(x_{T+1},c_{T+1}))\leq\\
    &\leq  2\sup_{w\in\cW}|f(w)-\overline{\mu}_T(w)|+\sup_{x\in \cX}\overline{\mu}_T(x,c_{T+ 1})-\overline{\mu}_T(x_{T+1},c_{T+1})\leq \\
    &\leq 10\left(\frac{11/3B+3R\sqrt{\gamma_T81F^2/13\log(6/\delta)}}{T}+\beta_2(\delta/12|\cU_T|)\sqrt{\frac{\gamma_ T}{T}}\right)+\frac{\gamma_T}{T}\frac{1}{\varepsilon}\frac{648BF^2}{13}\left(d\log\left(\frac{2592BF^2TL_f}{13\gamma_T}\varepsilon\right)+\log(3/\delta)\right) 
\end{align*}
Note that in the above expression  the randomness is over the previous contexts,rewards and random coins of the algorithm and not over the final context $c_{T+1}$. It thus follows that the claim holds uniformly over the entire context set $\cC$. We can thus write:
\begin{align*}
    &\mathop{\E}_{c_{T+1} \sim \kappa}\left[\sup_{x \in \cX} f(x,c_{T+1}) - f(\widehat{x}_T(c_{T+1}), c_{T+1})\right]\leq \\
    &\leq 10\left(\frac{11/3B+3R\sqrt{\gamma_T81F^2/13\log(6/\delta)}}{T}+\beta_2(\delta/12|\cU_T|)\sqrt{\frac{\gamma_ T}{T}}\right)+\frac{\gamma_T}{T}\frac{1}{\varepsilon}\frac{648BF^2}{13}\left(d\log\left(\frac{2592BF^2TL_fD_0}{13\gamma_T}\varepsilon\right)+\log(3/\delta)\right) 
\end{align*}

\end{proof}

\section{Appendix D. Kernel Trick}

IIn this section we present an efficient way to calculate the posterior statistics $\overline{\mu}_T, \overline{\sigma}$.

\begin{lemma}\label{appendix:kernel_trick}
    The estimator given in parametric form in Lemma\ref{approximation_lemma} can be equivalently written as:
    \begin{align}\label{appendixalgo:estimator_calc}
    \overline{\mu}_T(w)&=\frac{1}{\tau}\bfk^{\top}_{\mathbf{W}_T}(w)\bfY_T-\frac{1}{\tau}\bfk^{\top}_{\cZ}(w)\left(\bfK_{\cZ, \cZ}+K\tau\mathbf{I}_{\cZ}\right)^{-1}\bfK_{\cZ,\bfW_T}\bfY_T\\
    \overline{\sigma}(w)^2&=\frac{1}{\tau}\left(k(w,w)-\bfk^{\top}_{\cZ}(w)\left(K_{\cZ,\cZ}+K\tau\mathbf{I}_{\cZ}\right)^{-1}k_{\cZ}(w)\right)
    \end{align}

    Where $\bfK_{\cZ,\bfW_T}=\{k(a,b)\}_{a\in \cZ,b\in \bfW_T}, \bfK_{\cZ,\cZ}=\{k(a,b)\}_{(a,b)\in \cZ^2}$ and $K=\lceil T/\gamma_T\rceil$ .
\end{lemma}

\begin{proof}
    We introduce the shorthand notation $K= \lceil T/\gamma_T\rceil$. The parametric expression for $\overline{\mu}_T$ expression can be re-written as:
    \begin{align*}
        \overline{\mu}_T(w)&=\phi(w)^{\top}\tildeZ^{-1}\Phi_{\bfW_T}\bfY_T=K\phi(w)^{\top}(\Phi_{\cZ}\Phi_{\cZ}^{\top}+K\tau\mathbf{Id})^{-1}\Phi_{\bfW_T}\bfY_T=\\
        &= \frac{1}{\tau}\phi(w)^{\top}\left(\Phi_{\cZ}\Phi_{\cZ}^{\top}+K\tau\mathbf{Id}\right)^{-1}(\Phi_{\cZ}\Phi_{\cZ}^{\top}+K\tau\mathbf{Id}-\Phi_{\cZ}\Phi_{\cZ}^{\top})\Phi_{\bfW_T}\bfY_T=\\
        &=\frac{1}{\tau}k_{\mathbf{W}_T}(w)^{\top}\bfY_T-\frac{1}{\tau}\phi(w)^{\top}\left(\Phi_{\cZ}\Phi_{\cZ}^{\top}+K\tau\mathbf{Id}\right)^{-1}\Phi_{\cZ}\Phi_{\cZ}^{\top}\Phi_{\bfW_T}\bfY_T
    \end{align*}

    Where The second line follows from the reproducing property.  We next utilize a commonly applied identity \citep{Valko_et_al_2013,Vakili_Kernel_Simple_Regret} $\left(\Phi_{\cZ}\Phi_{\cZ}^{\top}+K\tau\mathbf{Id}\right)^{-1}\Phi_{\cZ}=\Phi_{\cZ}\left(\Phi_{\cZ}^ {\top}\Phi_{\cZ}+K\tau\mathbf{Id}\right)^{-1}$ . We can now write:
    \begin{align*}
        \overline{\mu}_T(w)&=\frac{1}{\tau}k_{\mathbf{W}_T}(g^{\top})\bfY_T-\frac{1}{\tau}\phi(w)^{\top}\Phi_{\cZ}\left(\Phi_{\cZ}^{\top}\Phi_{\cZ}+K\tau\mathbf{Id}\right)^{-1}\Phi_{\cZ}^{\top}\Phi_{\bfW_T}\bfY_T=\\
        &=\frac{1}{\tau}k_{\mathbf{W}_T}(w)^{\top}\bfY_T-\frac{1}{\tau}\phi(w)^{\top}\Phi_{\cZ}\left(\Phi_{\cZ}^{\top}\Phi_{\cZ}+K\tau\mathbf{Id}\right)^{-1}\Phi_{\cZ}^{\top}\Phi_{\bfW_T}\bfY_T=\\
        &=\frac{1}{\tau}k_{\mathbf{W}_T}(w)^{\top}\bfY_T-\frac{1}{\tau}k_ {\cZ}(w)^{\top}\left(\bfK_{\cZ, \cZ}+K\tau\mathbf{I}_{\cZ}\right)^{-1}\bfK_{\cZ,\bfW_T}\bfY_T
    \end{align*}

    Where $\bfK_{\cZ,\bfW_T}=\{k(a,b)\}_{a\in \cZ,b\in \bfW_T}$ and $\bfK_{\cZ,\cZ}=\{k(a,b)\}_{(a,b)\in \cZ^2}$.  We use a similar approach in calculating $\overline{\sigma}(w)$:
    \begin{align*}
        \overline{\sigma}^2(w)&=\phi(w)^{\top}\tildeZ^{-1}\phi(w)=\\
        &=\frac{1}{\tau}\phi(w)^{\top}\left(\Phi_{\cZ}\Phi_{\cZ}^{\top}+K\tau\mathbf{Id}\right)^{-1}(\Phi_{\cZ}\Phi_{\cZ}^{\top}+K\tau\mathbf{Id}-\Phi_{\cZ}\Phi_{\cZ}^{\top})\phi(w)=\\
        &=\frac{1}{\tau}k(w,w)-\frac{1}{\tau}\phi(w)^{\top}\left(\Phi_{\cZ}\Phi_{\cZ}^{\top}+K\tau\mathbf{Id}\right)^{-1}\Phi_{\cZ}\Phi_{\cZ}^{\top}\phi(w)=\\
        &=\frac{1}{\tau}k(w,w)-\frac{1}{K^2\tau}\phi(w)^{\top}\Phi_{\cZ}\left(\Phi_{\cZ}^{\top}\Phi_{\cZ}+K\tau\mathbf{Id}\right)^{-1}\Phi_{\cZ}^{\top}\phi(w)=\\
        &=\frac{1}{\tau}\left(k(w,w)-k_{\cZ}(w)^{\top}\left(K_{\cZ,\cZ}+K\tau\mathbf{I}_{\cZ}\right)^{-1}k_{\cZ}(w)\right)
    \end{align*}
\end{proof}

\end{document}